\documentclass[sn-nature]{sn-jnl}

\usepackage{graphicx}%
\usepackage{multirow}%
\usepackage{amsmath,amssymb,amsfonts}%
\usepackage{amsthm}%
\usepackage{mathrsfs}%
\usepackage[title]{appendix}%
\usepackage{xcolor}%
\usepackage{textcomp}%
\usepackage{manyfoot}%
\usepackage{booktabs}%
\usepackage{algorithm}%
\usepackage{algorithmicx}%
\usepackage{algpseudocode}%
\usepackage{listings}%
\usepackage{bm} 

\theoremstyle{thmstyleone}%
\theoremstyle{thmstyletwo}%

\theoremstyle{thmstylethree}%

\begin{document}

\title[Article Title]{Learning Continuous Network Emerging Dynamics from Scarce Observations via Data-Adaptive Stochastic Processes}







\author[1,2]{\fnm{Jiaxu} \sur{Cui}}

\author[2,3]{\fnm{Bingyi} \sur{Sun}}

\author[4]{\fnm{Jiming} \sur{Liu}}

\author*[1,2]{\fnm{Bo} \sur{Yang}}\email{ybo@jlu.edu.cn}

\affil[1]{\orgdiv{College of Computer Science and Technology}, \orgname{Jilin University}, \orgaddress{\country{China}}}

\affil[2]{\orgdiv{Key Laboratory of Symbolic Computation and Knowledge Engineering of Ministry of Education}, \orgname{Jilin University}, \orgaddress{ \country{China}}}

\affil[3]{\orgdiv{Public Computer Education and Research Center, Jilin University, Changchun, China}, \orgname{Jilin University}, \orgaddress{ \country{China}}}


\affil[4]{\orgdiv{Department of Computer Science}, \orgname{Hong Kong Baptist University}, \orgaddress{ \country{Hong Kong}}}


\abstract{
Learning network dynamics from the empirical structure and spatio-temporal observation data is crucial to revealing the interaction mechanisms of complex networks in a wide range of domains. 
However, most existing methods only aim at learning network dynamic behaviors generated by a specific ordinary differential equation instance, resulting in ineffectiveness for new ones, and generally require dense observations.
The observed data, especially from network emerging dynamics, are usually difficult to obtain, which brings trouble to model learning.
Therefore, how to learn accurate network dynamics with sparse, irregularly-sampled, partial, and noisy observations remains a fundamental challenge.
We introduce Neural ODE Processes for Network Dynamics (NDP4ND), a new class of stochastic processes governed by stochastic data-adaptive network dynamics, to overcome the challenge and learn continuous network dynamics from scarce observations.
Intensive experiments conducted on various network dynamics in ecological population evolution, phototaxis movement, brain activity, epidemic spreading, and real-world empirical systems,
demonstrate that the proposed method has excellent data adaptability and computational efficiency, and can adapt to unseen network emerging dynamics, producing accurate interpolation and extrapolation with reducing the ratio of required observation data to only about 6\% and improving the learning speed for new dynamics by three orders of magnitude.
}




\maketitle

\section{Introduction}\label{sec1}

Learning the underlying dynamics mechanism of complex systems has been a fundamental research \cite{Understanding2022}, 
which almost dominates the understanding for the formation and evolution of complex phenomenons with significant impacts in a wide range of domains, such as the global climate anomalies \cite{McCright2014}, the emergence of public opinion influencing political and economic pattern \cite{9772695}, as well as the epidemic outbreaks endangering the lives and health of billions of people \cite{LIU2020100354}.
As a superior modeling tool, network can be used to depict spatio-temporal interactions among system components, where network dynamics refers to the complex interaction mechanism \cite{Newman2011,2013Universality,GOSAK2018118}.
Therefore, learning network dynamics from empirical observations that capture the topological structure and spatio-temporal behavior is crucial to revealing the interaction mechanism and predicting the system behaviors.

In network dynamics, a tough issue is to learn network emerging dynamics \cite{Global_trends_emerging_infectious_diseases,GOSAK2018118,LOPPINI2018140}.
For instance, as a kind of impressive network emerging dynamics in public health \citep{HUIGANG202146}, the emerging infectious diseases (EID) \citep{Global_trends_emerging_infectious_diseases,1993Emerging} are usually caused by new or newly identified pathogenic microorganisms, e.g., SARS-CoV, MERS-CoV, and SARS-CoV-2 \citep{LU2020565}, giving rise to the novel exposed dynamics behaviors \citep{LIU2020100354}.
The task poses particular challenges, i.e., the \textit{novelty} and data \textit{sparsity}.
The novelty is reflected in the undiscovered and new formulas or parameters of the dynamics equations, resulting in the inapplicability of the existing models.
Due to the fact that data acquisition is costly and time-consuming \cite{NMI_spatio_temporal_electric}, attention to early transmission is inadequate \citep{DBLP:conf/aaai/PeiY0D18}, and broken sensors or damaged memory units may occur frequently \cite{Tang_Yao_Sun_Aggarwal_Mitra_Wang_2020},
observed data is usually scarce, irregularly-sampled, partial and noisy, bringing troubles to model learning.

Fortunately, artificial intelligence technologies have gradually led scientific discovery, breaking through protein structure analysis \cite{AlphaFold}, assisting in the design of new drugs \cite{Gentrl}, and advancing mathematics \cite{AImathematics}, which also provides a bright road for learning network dynamics.
Representative symbolic regression-based network dynamics studies have emerged \cite{Gao2022} to actively promote the development of complexity and network sciences.
However, the performance of these methods is strictly governed by expert knowledge, such as the selection of basis functions \cite{Mangan2017, Gao2022}.

With the rapid development of deep learning, researchers have developed effective neural simulators for networked systems \citep{DBLP:journals/ans/ZhangZLWTXZ19,MurphyNC-GNNSpain2021,Fritz2022,DBLP:journals/corr/abs-2211-10594,10.1145/3447548.3467385,Wang2022,10.1007/978-3-031-20047-2_13,kdd2023_Generalizing_Graph_ODE,NEURIPS2020_ba484941,NDCN,ijcai2023p242}.
Modeling ordinary differential equations (ODEs) on networks for characterizing propagation mechanisms through neural networks, these methods are able to automatically learn expressive network dynamics via a data-driven fashion.
Recently, few attempts to learn continuous-time network dynamics from irregularly-sampled partial observations have been made by elaborately designing a spatio-temporal transformer to aggregate observations \citep{NEURIPS2020_ba484941} and a graph neural network as the ODE function to model continuous network interactions \citep{NDCN,ijcai2023p242}.
Despite extending to promising dynamic graphs \citep{10.1145/3447548.3467385}, across environment learning \citep{kdd2023_Generalizing_Graph_ODE} and brain networks \citep{Wang2022},
they still struggle with the relatively dense observations required for sufficient training, i.e., at least 40\% $\sim$ 80\% of all data is required as observations for model learning \citep{NEURIPS2020_ba484941,NDCN,ijcai2023p242}.
Moreover, these methods only learn network dynamics behavior generated by a single network ODE instance, resulting in ineffective processing for new behaviors with few observations.

To address the aforementioned limitations, we introduce Neural ODE Processes for Network Dynamics (NDP4ND), a new family of stochastic processes governed by stochastic data-adaptive network dynamics, simultaneously considering temporal dynamics and interactions on topological structures.
Our analysis on various complex network dynamics scenarios, including mutualistic interaction dynamics in ecosystems, second-order phototaxis dynamics, brain dynamics, compartment models in epidemiology, and real-world global epidemic transmission systems, indicates that the NDP4ND has data adaptability and computational efficiency, and can accurately interpolate and extrapolate the unseen network emerging dynamics with sparse, irregularly-sampled, partial, and noisy observations.

\section{Results}\label{sec2}

\subsection{Problem statement of learning network dynamics}
A networked system consists of two key components \citep{Newman2011,2013Universality}: network dynamics $\pi_{\phi} \in \Pi$ and network topology $A \in \mathcal{A}$,
where 
$\Pi$ is a network dynamics space and $\mathcal{A}$ is a topological structure space.
Generally, $\pi_{\phi}$ can be considered as an ordinary differential equation (ODE) instance, i.e., $\frac{dX_l(t)}{dt} = \pi_{\phi}(\{X_j(t)\}_{j=1:n},A_{l,1:n})$,
where, $X_l(t)\in \mathbb{R}^{d}$ represents the states on node $l$ at time $t$, $d$ is the state dimension, and $n$ is the system size.
$A_{l,j}$ declares the influence or flow from node $j$ to $l$.
$\phi$ is an assignment of parameters in an ODE template $\pi$, which is a parameterized ODE characterizing continuous state changes. 
Assigning different values to the parameters in $\pi$ will result in corresponding network ODE instances that produce various system behaviors.
Given $\pi_{\phi}$ and $A$, system behaviors can be characterized by the spatio-temporal trajectory function $\omega \in \Omega$, where, $\Omega$ is a trajectory function space.
When the initial state $X_l(0)$ is known, $\omega$ can be obtained by calculating the state at any time $t$ by solving an ODE initial-value problem (IVP) using numerical integration \citep{Stickler2016}, i.e.,  $X_l(t)=X_l(0)+\int_{0}^{t}{\pi_{\phi}(\{X_{j}(\tau)\}_{j=1:n},A_{l,1:n})}d\tau$.
As $\pi_{\phi}$ and $A$ in trajectory function govern state changes together,
the distinguishing feature of network dynamics is that the updating of states on the networked system is influenced by neighbors.


Herein, we provide a problem statement for learning network dynamics. 
Denote an observation on spatio-temporal series produced by a networked system as a triple
$(t,l,X_{l}(t))$, where $t \in \mathbb{R}$, $l\in \{1,2,...,n\}$, and $n$ is the number of nodes on the interested system.
Formally, given a network topology $A$ of a complex networked system and a set of observations of the spatio-temporal series produced by the system, 
i.e., $\mathcal{D}_{\mathbb{C}}=
\{(t^{\mathbb{C}}_1,l^{\mathbb{C}}_1,X_{l^{\mathbb{C}}_1}(t^{\mathbb{C}}_1)),...,(t^{\mathbb{C}}_{N_{\mathbb{C}}},l^{\mathbb{C}}_{N_{\mathbb{C}}},X_{l^{\mathbb{C}}_{N_{\mathbb{C}}}}(t^{\mathbb{C}}_{N_{\mathbb{C}}}))\}
$, where $N_{\mathbb{C}}$ is the number of observations, learning network dynamics refers to forecasting the state on any node $l$ at any time $t$; that is, calculating a posterior distribution
$p(X_{l}(t)\vert t, l, \mathcal{D}_{\mathbb{C}},A)$ based on empirical observations capturing the structure and spatio-temporal behavior.
Let $T_o$ denote the maximum value of all observed times in $\mathcal{D}_{\mathbb{C}}$.
When all $t$ in the prediction task is less than or equal to $T_o$, it becomes an interpolation task, and while all $t$ in the prediction task is greater than $T_o$, then it is an extrapolation task.
When $\mathcal{D}_{\mathbb{C}}$ comes from a network dynamics with a new ODE instance and its involving observations are sparse, irregularly-sampled, partial, and noisy, the problem could be viewed as learning network emerging dynamics that we are concerned about.

\begin{figure}[htbp]%
\centering
\includegraphics[width=1.\textwidth]{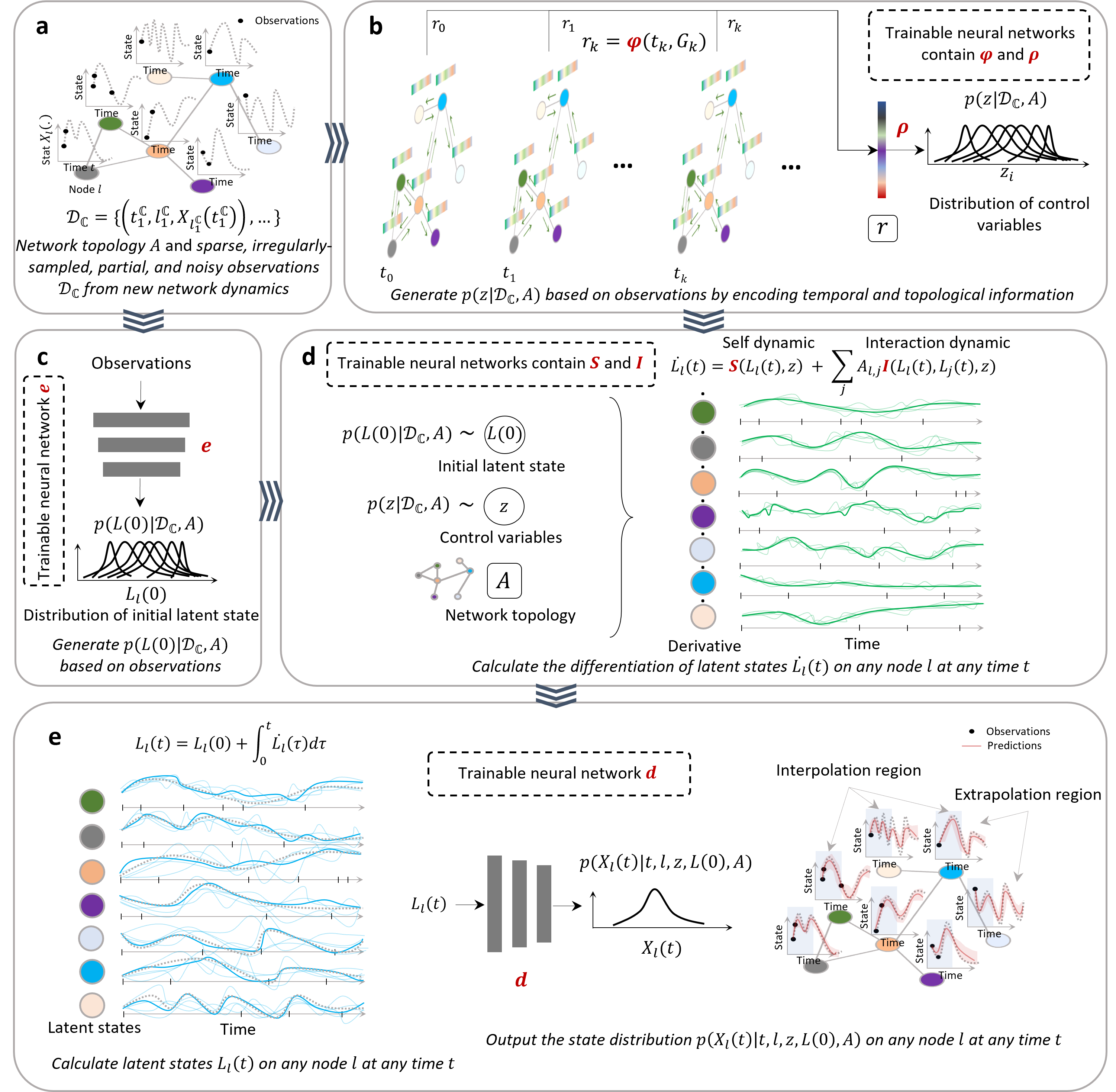}
\caption{
\textbf{The overall workflow of the NDP4ND.}
\textbf{a}, Empirical data from network dynamics, including network topology $A$ and state observations $\mathcal{D}_{\mathbb{C}}$.
\textbf{b}, Generating the distribution over the global random vector $z$, i.e., $p(z|\mathcal{D}_{\mathbb{C}},A)$ by encoding temporal and topological information.
\textbf{c}, Generating the distribution over the initial states in latent space, i.e., $p(L(0)|\mathcal{D}_{\mathbb{C}},A)$.
\textbf{d}, After obtaining the distributions of $z$ and $L(0)$, we perform dynamic propagation in the hidden space based on the universal skeleton of network dynamics and give the differentiation of hidden states on node $l$ at time $t$, i.e.,
$
    \dot{L}_{l}(t)
$.
\textbf{e}, Calculating the hidden states ${L}_{l}(t)$ and then output the state distribution on any node $l$ at any time $t$, i.e., $p(X_l(t)|t,l,z,L(0),A)$.
Note that neural networks ($\bm{\varphi}$, $\bm{\rho}$, $\bm{e}$, $\bm{S}$, $\bm{I}$, $\bm{d}$) are to parameterize the stochastic processes and achieve adaptive learning from data.
}
\label{flow}
\end{figure}

\subsection{Overview of neural ODE processes for network dynamics}


As noted, most of the existing methods only learn a specific trajectory sample in trajectory function space and are tricky to generalize to other trajectory samples from new network ODE or emerging empirical systems.
On the contrary,
we expect our method to represent the trajectory function space and adapt to unseen tasks with few observations.

To achieve this goal, our core idea is to use stochastic processes to model the trajectory function space $\Omega$.
Although one can independently model a stochastic process for the trajectory per node, due to the fact that nodes in the networked system interact with others, there are often important correlations between nodes.
We, thus, establish a general stochastic process to model the complex trajectory function space, while considering the correlation between the states of adjacent nodes.
That is, we model a collection $\{X_l(t)\}$ as a stochastic process $\mathcal{F}$, where the node state $X_l(t)\in\mathbb{R}^d$ is a random variable.
In principle, as time and space (network) expand, there are infinitely many such random variables in the networked system.

Specifically, we propose Neural ODE Processes for Network Dynamics (NDP4ND), a new family of stochastic processes parameterized by neural networks, simultaneously considering temporal dynamics and node interactions on topological structures, to approximate random spatio-temporal trajectory functions on networks.
By properly combining the universal skeleton of network dynamics \citep{2013Universality}, latent neural ODE \citep{NEURIPS2018_69386f6b}, and neural processes \citep{DBLP:journals/corr/abs-1807-01622} to build the family of distributions suitable for network dynamics modeling, the NDP4ND maintains their respective strengths, including broad network dynamics representation, continuous temporal changes, and data-adaptation.
In terms of theoretical existence, we state that the constructed family of finite-dimensional distributions satisfies the exchangeability and consistency conditions (Proposition 1 in the Section A.2 of the supplementary information).
According to the Kolmogorov Extension Theorem \citep{oksendal2003}, these conditions are sufficient to define a stochastic process. 
In other words, the stochastic process we established exist, and the family of distributions is just its.
Fig.~\ref{flow} shows the overall workflow of the NDP4ND.
After obtaining the empirical data from network emerging dynamics (Fig.~\ref{flow}a), including network topology $A$ and sparse, irregularly-sampled, partial, and noisy observations $\mathcal{D}_{\mathbb{C}}$,
we can infer a distribution over the global random vector $z$ by encoding temporal and topological information (Fig.~\ref{flow}b), i.e., $p(z|\mathcal{D}_{\mathbb{C}},A)$.
By integrating observations with the same time and combining them with the topology to form a graph $G(t)$,
the observation set $\mathcal{D}_{\mathbb{C}}$ can be then transformed into a graph set $\{(t_1,G(t_1)),...,(t_K,G(t_K))\}$.
We use $\bm{\varphi}$ to produce a representation for each pair in the graph set by $r_k = \bm{\varphi}(t_k,G(t_k))$, and then use $\bm{\rho}$ to give the distribution of $z$ based on an aggregation of all representations, i.e., $r=agg(\{r_k\}_{k=1}^{K})$, where, $agg$ is an aggregation operation with permutation invariance.
Simultaneously, the distribution over the initial states in latent space, i.e., $p(L(0)|\mathcal{D}_{\mathbb{C}},A)$, can be generated by a neural network $\bm{e}$ based on empirical data (Fig.~\ref{flow}c).
After inferring the distributions of $z$ and $L(0)$, we perform dynamic propagation in the hidden space based on the universal skeleton of network dynamics and give the differentiation of hidden states on node $l$ at time $t$ (Fig.~\ref{flow}d), i.e.,
$
    \dot{L}_{l}(t)=\bm{S}(L_{l}(t),z)+\sum_{j=1}^{n}A_{l,j}\bm{I}(L_{l}(t),L_{j}(t),z)
$,
where $\bm{S}$ and $\bm{I}$ are neural networks, representing self dynamics and interaction dynamics respectively, and global random vector $z$ can stochastically control the derivative function of network dynamics.
Then, we can calculate the hidden states ${L}_{l}(t)$ by $L_{l}(t) = L_{l}(0)+\int_{0}^{t}{\dot{L}_{l}(\tau)}d\tau$ and 
use neural network $\bm{d}$ to output the state distribution on any node $l$ at any time $t$, i.e., $p(X_l(t)|t,l,z,L(0),A)$ (Fig.~\ref{flow}e).
Note that neural networks ($\bm{\varphi}$, $\bm{\rho}$, $\bm{e}$, $\bm{S}$, $\bm{I}$, $\bm{d}$) are to parameterize the stochastic processes and achieve adaptive learning from data.
Please refer to Methods for detailed modeling of the stochastic process.

\subsection{Learning various network dynamics}

We present the results and analysis on various network dynamics in ecological population evolution, phototaxis movement, brain activity, and epidemic spreading.
Please refer to the Section B of the supplementary information for details of network dynamics and experimental settings.

\textbf{Mutualistic interaction dynamics in ecosystems.}
Mutualistic interaction dynamics among species in ecology capture the abundance $X_i(t)$ of species $i$ by integrating the mutualistic interactions from neighboring ecosystems, such as the incoming migration and negative growth \citep{JianxiNature2016}.
We consider five network topologies in this scenario, i.e., grid, random, power-law, small world, and community networks. 
By assigning randomly sampled values from a certain range to the tunable parameters $\phi$ in the network ODE of mutualistic interaction dynamics, we can obtain various ODE instances.
Then, based on different ODE instances, topologies, and initial states sampled from the distribution $\mathcal{I}$, we randomly produce $N_{Tr}$ network dynamics trajectories as the training set and $N_{Te}$ trajectories as the testing set.
Please note that the sampling ODE instances used to generate the testing set have not yet appeared in the training set.
That is to say, all testing ODEs can be considered as network emerging dynamics.
The irregular-and-sparse-sampled partial observations from each testing task are to trigger predictions.
We compared with the state-of-the-art neural dynamics methods on networks for interpolation and extrapolation tasks in terms of prediction error and similarity, i.e., the LG-ODE \citep{NEURIPS2020_ba484941}, NDCN \citep{NDCN}, and DNND \citep{ijcai2023p242}, who can deal with the irregular-sampled partial observations and learn continuous network dynamics.
The power comes from using ODE to model the continuous change in hidden space \citep{NEURIPS2020_ba484941, NDCN} or state space \citep{ijcai2023p242}.
The average ratio of observations in the testing network ODEs is $5.81\%$.
We see that the compared neural dynamics methods unsurprisingly underestimate the testing ODEs and extremely unstable because of the low sampling density, while our NDP4ND reaches the best results for both interpolation and extrapolation tasks, showing its effectiveness in learning on sparse and irregular-sampled observations from new ODEs (Extended Data Table~\ref{table_mutu})
We showcase the testing results of each method on a network ODE instance (Fig.~\ref{fig_mutu}a-e).
\begin{figure}[htbp]%
\centering
\includegraphics[width=1.\textwidth]{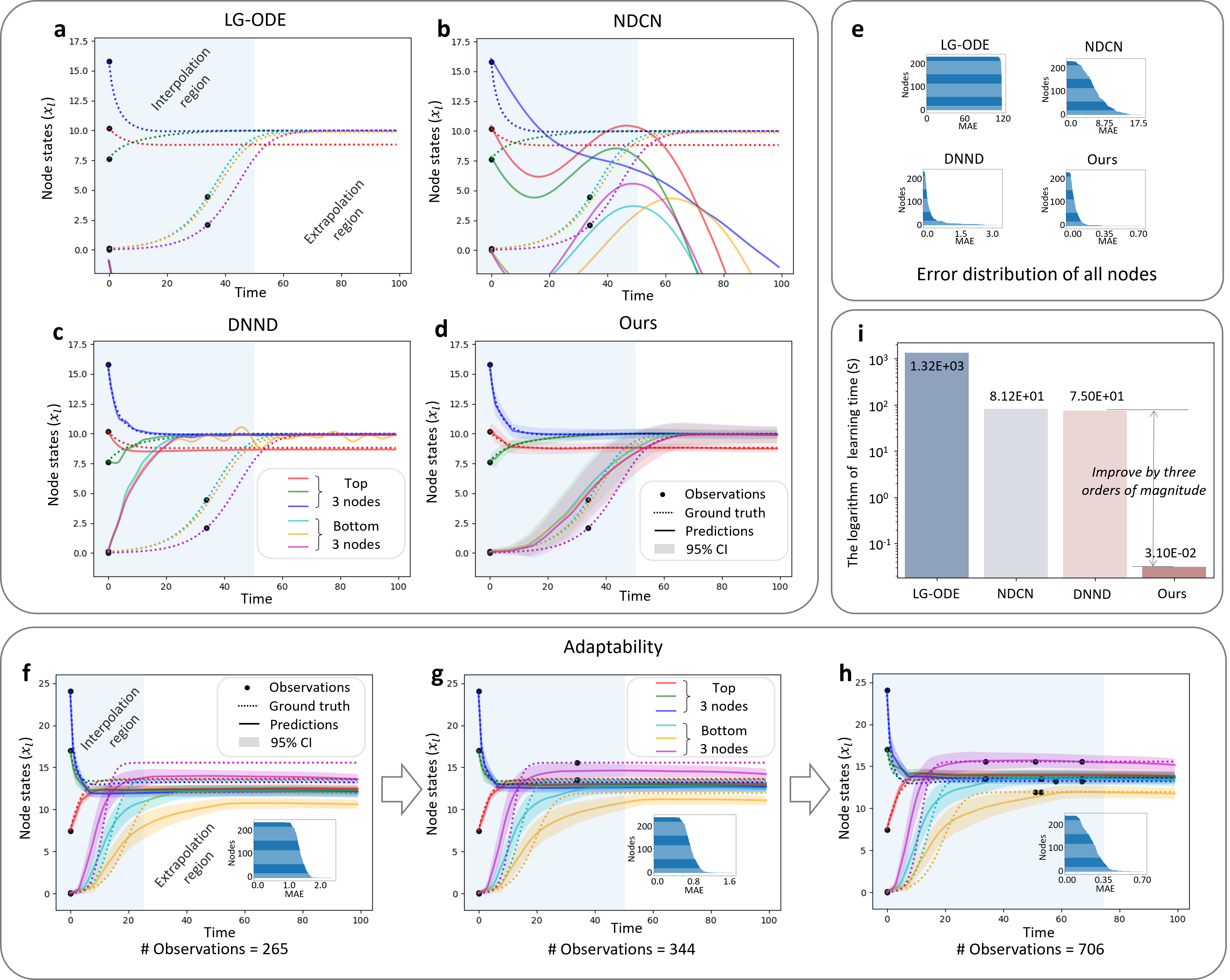}
\caption{
The interpolation and extrapolation results on \textit{mutualistic interaction dynamics}.
\textbf{a-d}, The testing results of the best-performing and worst-performing three nodes of the LG-ODE, NDCN, DNND, and our NDP4ND.
The ratio of observations in this testing network ODE is $3.06\%$.
The number of nodes and the maximum value of all observed times ($T_o$) are 225 and 50, respectively.
Our NDP4ND has the best performance in both interpolation and extrapolation, and ultimately stabilizes at two population values. 
\textbf{e}, The distributions of Mean Absolute Error (MAE) between predictions and ground truth for all nodes, demonstrating the high-precision predictions from our method.
\textbf{f-h}, Our NDP4ND can directly utilize newly arrived data and does not require any retraining, with adaptively rectifying predictions toward ground truth as observations increase.
Note that the subfigures only show the observation data on six nodes instead of all nodes.
\textbf{i}, The average learning time for all testing network ODEs, showing that our NDP4ND improves the learning speed for new dynamics by three orders of magnitude.
}\label{fig_mutu}
\end{figure}
The LG-ODE has not learned any effective network dynamics from a few observations (Fig.~\ref{fig_mutu}a) and the NDCN only fits well at the observed locations, while the result in the extrapolation region has a significant gap from the ground truth (Fig.~\ref{fig_mutu}b). 
Although the DNND can learn relatively stable network dynamics, it still underestimates the task (Fig.~\ref{fig_mutu}c).
Our NDP4ND has high-precision predictions for interpolation and extrapolation, and ultimately stabilizes at two population values (Fig.~\ref{fig_mutu}d-e).
Since the LG-ODE, NDCN, and DNND are aimed at learning a single network dynamics instance, it is necessary to retrain for each testing task on sampled observations.
Especially, when new observation data from testing ODEs arrives, they need to be retrained to use newly arrived observations.
On the contrary, ours does not require any retraining and can adaptively rectify predictions toward ground truth (Fig.~\ref{fig_mutu}f-h), benefiting from modeling stochastic processes.
By calculating the average learning time for processing all testing network ODEs, we see that our NDP4ND significantly improves learning speed by three orders of magnitude compared to the latest method, thanking to the adaptability to the data (Fig.~\ref{fig_mutu}i).

\begin{figure}[t]%
\centering
\includegraphics[width=1.\textwidth]{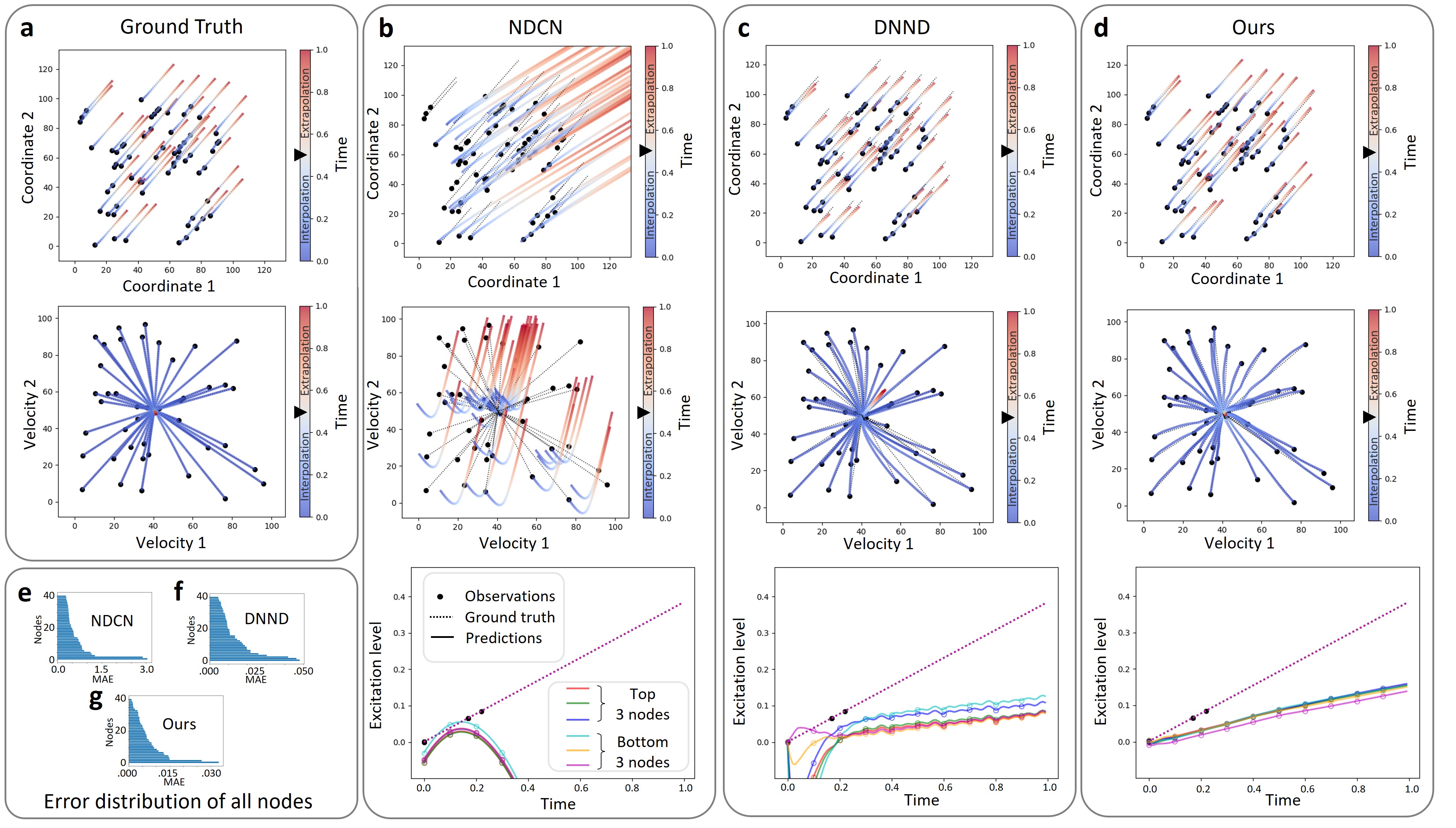}
\caption{
The interpolation and extrapolation results on \textit{phototaxis dynamics}.
\textbf{a-d}, The ground truth and testing results of the NDCN, DNND, and our NDP4ND.
The ratio of observations in this testing network ODE is $3.23\%$.
The number of nodes and the maximum value of all observed times ($T_o$) are 40 and 0.5, respectively.
The predictive five states, including coordinate 1, coordinate 2, velocity 1, velocity 2 and excitation level, are reported.
The NDCN has not learned effective network dynamics from a few observations.
The DNND fits well at the observed locations, but its velocity eventually converges to the wrong solution.
Our NDP4ND ultimately stabilizes at the correct velocity and the learned dynamics are relatively stable and close to the ground truth.
\textbf{e-g}, The distributions of Mean Absolute Error (MAE) between predictions and ground truth for all nodes, demonstrating the high-precision predictions from our method.
}\label{fig_phototaxis}
\end{figure}

\textbf{Second-order phototaxis dynamics.}
Phototaxis dynamics is a second-order system simulating the dynamics of phototactic bacteria towards a fixed light source \citep{phototaxis}, which can be applied to movement of bacteria towards food sources \citep{doi:10.1073/pnas.1822012116} and emergency evacuation modeling \citep{emergency_airplane_evacuations}.
The system performs bacteria's movements through the excitation level of bacteria to the light and their interactions with each other \citep{doi:10.1073/pnas.1822012116}.
The complete graph is used to describe all interactions between any two bacteria.
The average ratio of observations in the testing network ODEs is $4.65\%$.
We do not compare the LG-ODE due to its poor performance.
We see that the our NDP4ND has the optimal results for both interpolation and extrapolation tasks, which means it can learn effective network dynamics (Extended Data Table~\ref{table_phototaxis}).
The predictive five states on an instance, including coordinate 1, coordinate 2, velocity 1, velocity 2 and excitation level, are reported in Fig.~\ref{fig_phototaxis}.
The NDCN cannot learn effective network dynamics from a few observations (Fig.~\ref{fig_phototaxis}b).
The DNND fits well at the observed locations, but its velocity eventually converges to the wrong solution (Fig.~\ref{fig_phototaxis}c).
Our NDP4ND ultimately stabilizes at the correct velocity, can learn relatively stable and accurate network dynamics (Fig.~\ref{fig_phototaxis}d), and has lowest predictive errors (Fig.~\ref{fig_phototaxis}e-g).

\begin{figure}[t]%
\centering
\includegraphics[width=1.\textwidth]{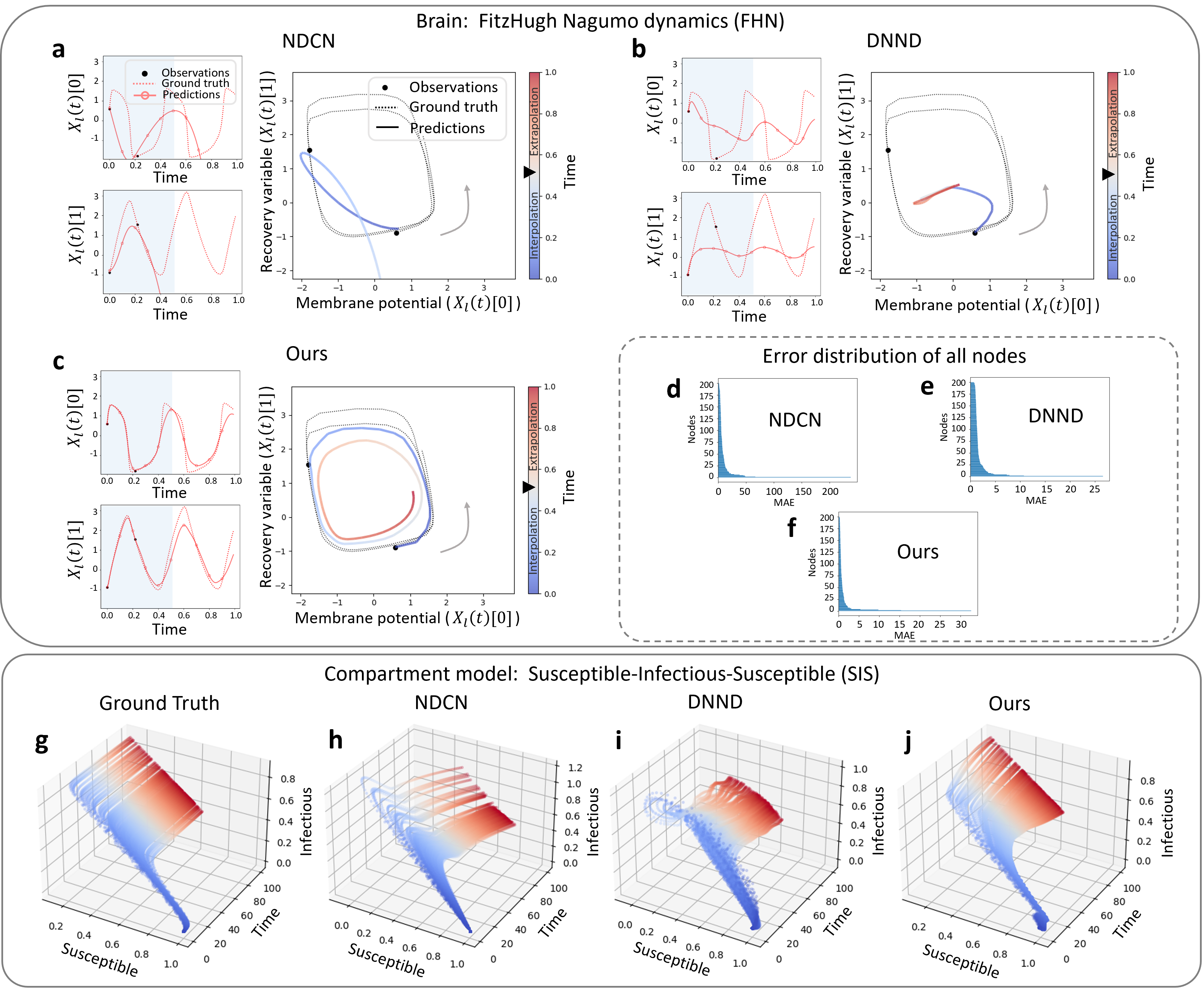}
\caption{
The interpolation and extrapolation results on \textit{brain dynamics} and \textit{susceptible-infectious-susceptible (SIS)}.
\textbf{a-c}, The testing results of the NDCN, DNND, and our NDP4ND on brain dynamics.
The ratio of observations in this testing network ODE is $3.23\%$.
The number of nodes and the maximum value of all observed times ($T_o$) are 200 and 0.5, respectively.
Two-dimensional states, i.e., membrane potential ($X_l(t)[0]$) and recovery variable ($X_l(t)[1]$), are reported.
Our NDP4ND can learn effective dynamics that reveals the expected periodicity in both interpolation and extrapolation regions. 
\textbf{d-f}, The distributions of Mean Absolute Error (MAE) between predictions and ground truth for all nodes on brain dynamics, demonstrating the high-precision predictions from our method.
\textbf{g-j}, The ground truth and testing results of the NDCN, DNND, and our NDP4ND on SIS dynamics.
The number of nodes and the maximum value of all observed times ($T_o$) are 200 and 50, respectively.
The time series predicted by our method are closest to the ground truth and meet the constraint of Susceptible+Infectious=1.
}\label{fig_brain_sis}
\end{figure}

\textbf{Brain dynamics.}
We also apply our method to brain activities governed by the FitzHugh-Nagumo dynamics \citep{FHN2020}, which has a quasi-periodic characteristic.
Topological structures with power-law distribution are used to simulate the interactions among brain functional areas \citep{Gao2022}.
The average ratio of observations in the testing network ODEs is $5.07\%$.
We see that our NDP4ND has lowest predictive error and highest similarity (Extended Data Table~\ref{table_brain}).
From a specific testing ODE instance, we see that the NDCN fits well at the observed locations, but cannot learn the inherent periodic behavior from a few observations (Fig.~\ref{fig_brain_sis}a) and the DNND still underestimates the task (Fig.~\ref{fig_brain_sis}b).
As a comparison, the predictive results produced by our NDP4ND exhibit periodic behavior, being a support of learning effective network dynamics (Fig.~\ref{fig_brain_sis}c).
The lowest MAE error also confirms the effectiveness of our method's prediction (Fig.~\ref{fig_brain_sis}d-f).

\textbf{Compartment models in epidemiology.}
Compartment models \citep{doi:10.1287/educ.1100.0075,individualSIR} are often used to simulate the spread of epidemics on networks. 
Herein, we conduct experiments on three commonly used compartment models: susceptible-infectious-susceptible (SIS), susceptible-infectious-recovered (SIR), and susceptible-exposed-infectious-susceptible (SEIS).
The observed ratio of testing tasks is around 6.0\%.
Although our method is slightly inferior to the NDCN in the interpolation region in some scenarios, the results in the extrapolation region are far superior to the compared methods (Extended Data Table~\ref{table_compartment}), showing its power to extrapolate and cover various epidemic scenarios.
An instance of the SIS dynamics is shown in Fig.~\ref{fig_brain_sis}g.
Both the NDCN and DNND cannot produce the effective predictions (Fig.~\ref{fig_brain_sis}h).
Especially, the results generated by the DNND does not meet the constraint of $\text{Susceptible}+\text{Infectioous}=1$ (Fig.~\ref{fig_brain_sis}i).
As a comparison, the shape of the time series predicted by our method is closest to the ground truth and the results are on the plane of $\text{Susceptible}+\text{Infectioous}=1$ (Fig.~\ref{fig_brain_sis}j).
More instances can be found in Extended Data Fig.~\ref{fig_sir_seis}.

\begin{figure}[t]%
\centering
\includegraphics[width=.7\textwidth]{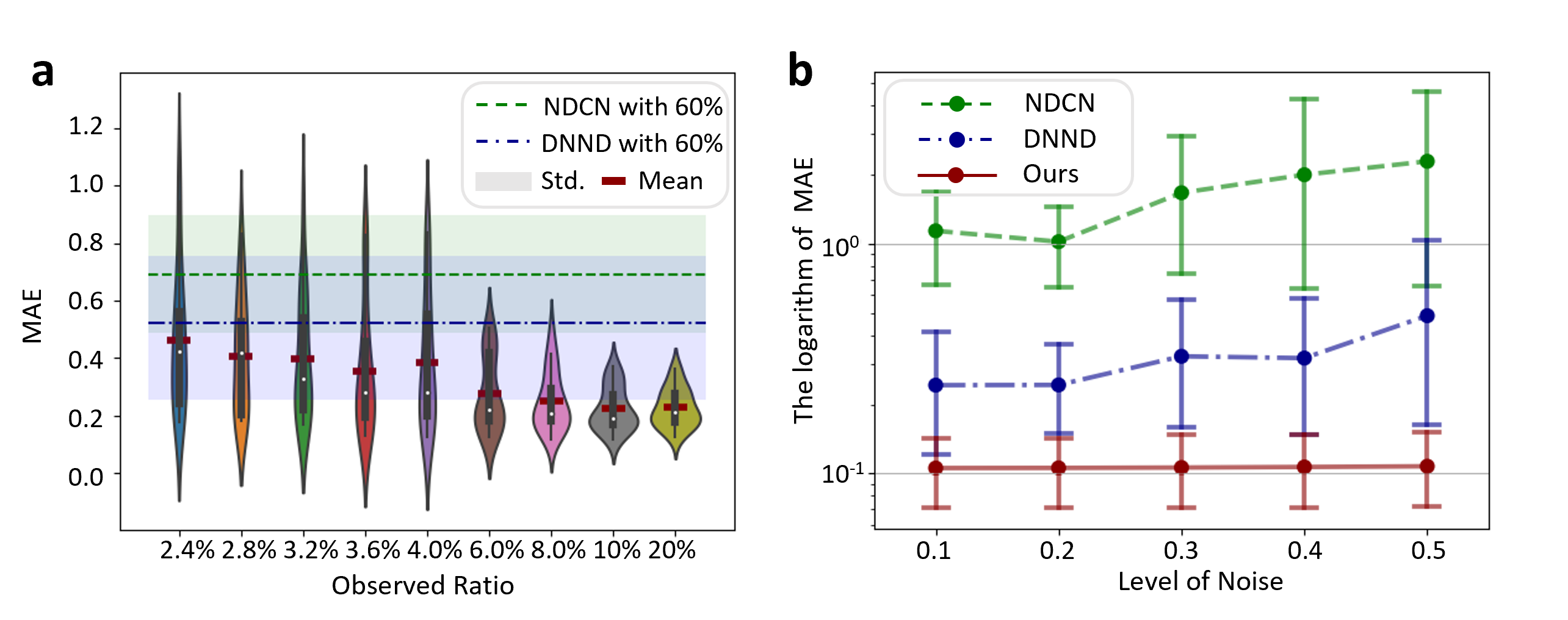}
\caption{
\textbf{a}, The testing results on mutualistic interaction dynamics at various observed ratios.
As the observation density increases, the predictive error of the NDP4ND appears a decreasing trend, where the prediction error occurs a cliff like decline when the observed ratio is around 6\%, reaching around 0.3.
By contrast, the prediction errors of the LG-ODE, NDCN, and DNND at an observed ratio of 60\% are 3.99$\pm$3.59, 0.70$\pm$0.20, and 0.52$\pm$0.24, respectively.
\textbf{b}, The testing results on mutualistic interaction dynamics at various noise levels.
As the noise level increases, the predictive errors shows an increasing trend, demonstrating that the introduction of noise indeed weakens predictions.
Benefiting from uncertainty modeling, our NDP4ND is relatively stable and has significantly outperformed others.
}\label{fig_sparse_noise}
\end{figure}

\textbf{Ablation studies.}
By comparing existing neural process models \citep{DBLP:journals/corr/abs-1807-01622,DBLP:conf/iclr/NorcliffeBDML21} and variations of our method, the results show that our proposed method and its variants outperform the neural process models (Extended Data Table~\ref{table_mutu_NP}), demonstrating that modeling the interactions of temporal dynamics is necessary.
Moreover, the NDP4ND achieves better results than its variants, which confirms the effectiveness of our computational architecture design and explicitly modeling the ODE flow describing network dynamics.

\begin{figure}[t]%
\centering
\includegraphics[width=1.\textwidth]{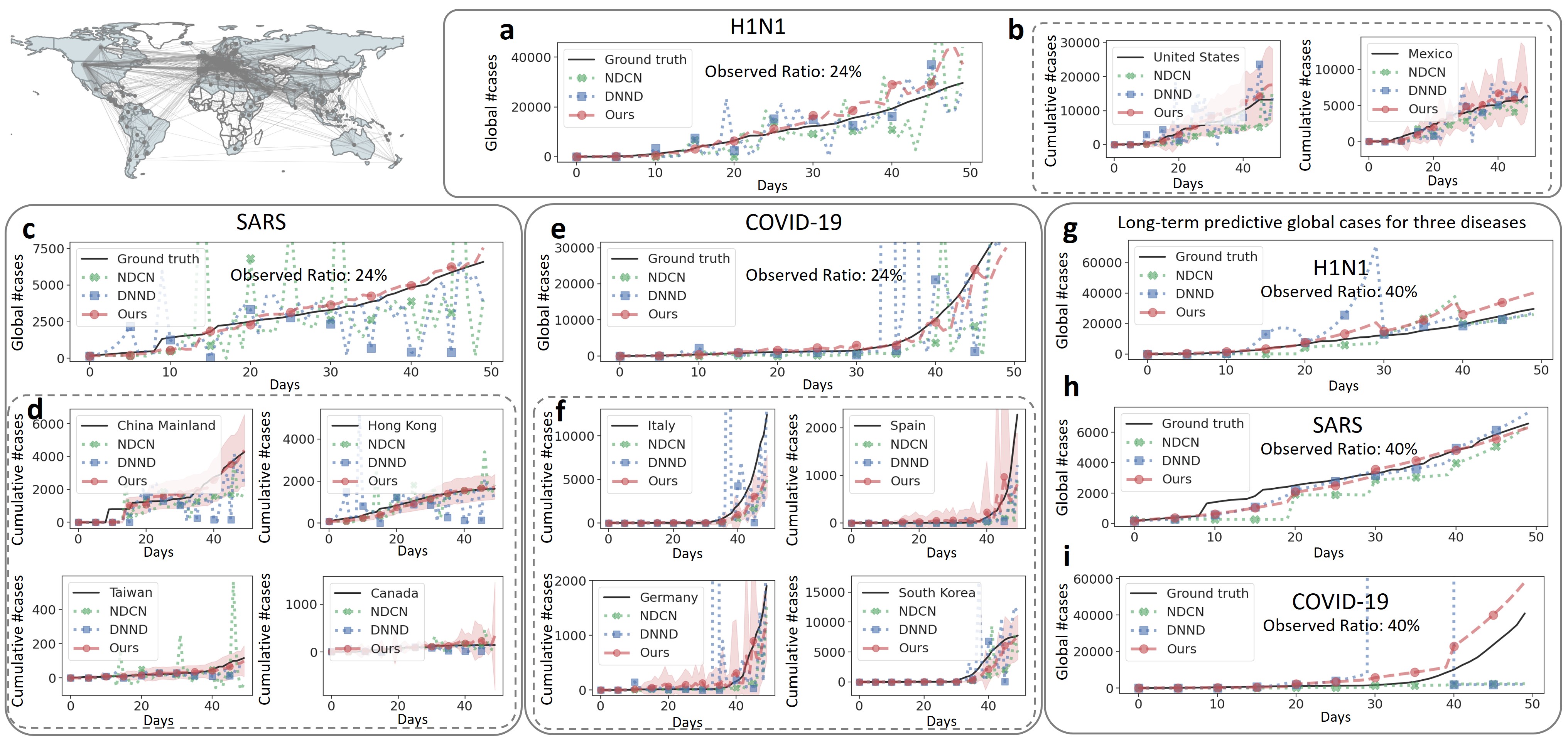}
\caption{
The testing results on real-world global epidemic transmission.
Note that, testing on H1N1 can be seen as predicting re-emerging infectious diseases, but testing on SARS and COVID-19 as predicting emerging infectious diseases.
\textbf{a-f}, The short-term forecasting results of the methods for three different diseases.
From the global case number (\textbf{a,c,e}), our NDP4ND has smoother and better results than others. 
The predictions on the case number in major countries or regions (\textbf{b,d,f}) also supports this performance.
\textbf{g-i}, The long-term forecasting results of the methods for three different diseases.
our NDP4ND achieves relatively better results, especially for emerging infectious disease COVID-19.
Other methods either fail to predict obvious outbreak trends (NDCN) or predict terrifying drastic erroneous fluctuations (DNND at 30-40 days) from sparse observations, while our method is closer to real epidemic outbreak data and has similar trends.
}\label{fig_real}
\end{figure}

\subsection{Learning network dynamics from sparse and noisy observations}

We also test our method's tolerance for the observed ratio.
By predicting system behaviors for 20 new sampled network ODE instances per observed ratio on mutualistic interaction dynamics, Fig.~\ref{fig_sparse_noise}a shows that as the observation density increases, the predictive error of the NDP4ND appears a decreasing trend, 
where the prediction error occurs a cliff like decline when the observed ratio is around 6\%, reaching around 0.3.
By contrast, the prediction errors of the LG-ODE, NDCN, and DNND at an observed ratio of 60\% are 3.99$\pm$3.59, 0.70$\pm$0.20, and 0.52$\pm$0.24, respectively.
This indicates that our NDP4ND's pretraining on a large amount of relevant data is helpful in significantly reducing the required observations when learning new network dynamics.
To test the robustness of our NDP4ND for various noise levels, we add white Gaussian disturbances to observations.
Our method is tested on sampled 20 new network ODE instances per level.
Fig.~\ref{fig_sparse_noise}b shows that noise indeed weakens predictions.
Benefiting from uncertainty modeling, our NDP4ND is relatively stable and has significantly outperformed the NDCN and DNND.

\subsection{Learning from empirical systems}
As impressive network dynamics, emerging and re-emerging infectious diseases are significant burdens on global economies and public health \citep{WHO_estimates,Global_trends_emerging_infectious_diseases}.
We collect daily global spreading data on three real diseases, including H1N1 \citep{dataweb}, SARS \citep{dataweb}, and COVID-19 \citep{covid19data}, and the worldwide airline network retrieved from OpenFights as a directed and weighted empirical topology to build the empirical systems of real-world global epidemic transmission.
Only early data before government intervention, i.e., the first 50 days, is considered here to keep the spread features of the disease itself.
We allow H1N1 data as training set according to various time intervals and test on all three diseases.
As a result, testing on H1N1 can be seen as predicting re-emerging infectious diseases, as the method trained on the same data but testing on SARS and COVID-19 as predicting emerging infectious diseases.
For short-term forecasting, models use partial observations from the previous 5 days to predict the case number on the next day (Fig.~\ref{fig_real}a-f).
Long-term forecasting requires those from the previous 10 days to predict the next 10 days (Fig.~\ref{fig_real}g-i).
We see that our NDP4ND achieves smoother and better results than the NDCN and DNND in both short-term and long-term predictions.
This demonstrates its expressive power for empirical data and its potential for processing real-world emerging and re-emerging tasks.

\section{Discussion}\label{sec3}

We propose the NDP4ND, a new class of stochastic processes governed by stochastic data-adaptive network dynamics, to represent the trajectory function space generated by network dynamics to provide a reasonable solution for learning network emerging dynamics from sparse, irregularly-sampled, partial, and noisy data.
Intensive experiments demonstrate the effectiveness of our method in adapting to many simulated and real-world tasks.
However, the accuracy of the current method is still limited by the approximate numerical integration in the ODE flow. 
Therefore, how to alleviate error accumulation is an outstanding problem, especially for the neural-ODE-based methods \citep{NEURIPS2018_69386f6b}.
It may be a promising attempt to map raw dynamics to linear ODE through an invertible neural network \citep{DBLP:conf/icml/ZhiLOBR22}, and then speed up the integration process or directly calculating the closed-form solution of latent linear ODE \citep{DBLP:journals/natmi/HasaniLALRTTR22}.

Real-world complex networked systems are ubiquitous, ranging from celestial movement to epidemic spreading.
Such a large number of scenarios have their own unique characteristics, such as the periodicity of brain activity data \citep{FHN2020} and the equilibrium of thermal dynamics \citep{Luikov2012}.
Therefore, how to beyond our method and to integratively and systematically represent the trajectory function space produced by abundant network dynamics scenarios and build a pretrained network dynamics model with richer expressive power for more diverse downstream tasks still remains an open question, such as heterogeneous \citep{Controlling2023}, high-order \citep{bianconi_2021}, and controllable network dynamics \citep{Controlling2021}.

\section{Methods}\label{sec4}

\subsection{Preliminaries}

\textbf{Universal skeleton of network dynamics.}
With appropriate choices of the functions $f$ and $g$, the following skeleton of the dynamical equation can describe a broad range of complex networked systems \citep{2013Universality}, 
$
\dot{X}_l(t)=f(X_l(t))+\sum_{j}{A_{l,j}g(X_l(t),X_j(t))}
$,
where, $f$ is self dynamics and $g$ is interaction dynamics.
The former quantifies the impact of one's own state on change, while the latter quantifies the impact of neighboring states on change (Extended Data Fig.~\ref{fig_skeleton}a).
Note that the distinguishing feature of network dynamics is that the updating of states on the networked system is influenced by adjacent nodes.
The skeleton has been successfully applied to guide the reconstruction of governing equations of network dynamics \citep{Gao2022}.

\textbf{Neural ODE}.
Neural ODEs refer to a technique that parameterizes state differentiation via neural networks \citep{NEURIPS2018_69386f6b},
i.e.,
$
\dot{X}(t)={NN}(X(t),t)
$, and uses the adjoint method \citep{NEURIPS2018_69386f6b} to solve the initial-value problem (IVP) to ensure that the model can effectively perform backpropagation to learn parameters of neural networks.
Latent neural ODEs usually encode the original state into the hidden space, parameterize the differentiation of the hidden state, and then obtain the future state by performing the integration and decoding operations, i.e.,
$
L(0)=enc(X(0))
$,
$
L(t')=L(0)+\int_{t_0}^{t'}{NN(L(t),t)}dt
$,
$
X(t')=dec(L(t'))
$,
where, $X(0)$ is a shorthand for the initial state $X(t_0)$, $L(0)$ is the initial hidden state, and $t'$ is the target time.
Note that $enc$ and $dec$ can also be neural networks.
It has been observed that the way to learn dynamics after mapping into hidden space has excellent performance \citep{DBLP:conf/iclr/NorcliffeBDML21,NDCN} and can express more complex original dynamics \citep{10.5555/3495724.3496220}. 
In network dynamics, effective network encoding can reduce complex spatio-temporal patterns to simple \citep{DBLP:journals/natmi/FloryanG22} and homogeneous dynamic patterns \citep{doi:10.1126/science.1245200}, e.g., the global spread of epidemics and opinions.


\textbf{Neural processes}.
Neural Processes (NPs) are neural network-parameterized stochastic processes \citep{DBLP:journals/corr/abs-1807-01622}, which can model a random function $F: \mathcal{T} \rightarrow \mathcal{X} \subseteq\mathbb{R}^{d}$, i.e., 
$F(.)=NN(.,z)$.
The neural network ($NN$) determines the internal statistical behavior of the stochastic process and the global latent random vector $z$ controls the uncertainty, i.e. generating different functions.
Note that while the input and output domains of the random function $F$ could be arbitrary, we are more interested in time series on networks, hereby using time as the input and state as the output.
Given the context set $\mathcal{D}_{\mathbb{C}}=\{(t_{1:N_{\mathbb{C}}}^{\mathbb{C}},X(t_{1:N_{\mathbb{C}}}^{\mathbb{C}}))\}$ and target set $t_{1:N_{\mathbb{T}}}^{\mathbb{T}}, X(t_{1:N_{\mathbb{T}}}^{\mathbb{T}})$, the generative process is
$
p(X(t_{1:N_{\mathbb{T}}}^{\mathbb{T}}),z\vert t_{1:N_{\mathbb{T}}}^{\mathbb{T}},\mathcal{D}_{\mathbb{C}})=p(z\vert \mathcal{D}_{\mathbb{C}})\prod_{i=1}^{N_{\mathbb{T}}}\mathcal{N}(X(t_{i}^{\mathbb{T}})\vert NN(t_i^{\mathbb{T}},z),\sigma^2)
$.
Due to the strong data adaptability and computational efficiency, its advanced versions have been developed, incorporating application-specific inductive biases of attention mechanism \citep{DBLP:conf/iclr/KimMSGERVT19} and translation equivariance \citep{NEURIPS2020_5df0385c}.
A small number of studies used graphs to model dependencies among inputs \citep{DBLP:journals/corr/abs-1812-05212} and incorporated GCNs into the NP architecture to improve the performance of traditional graph learning tasks, including node classification \citep{cangea2022message}, edge imputation \citep{DBLP:journals/corr/abs-1902-10042}, and link predictions \citep{9746010}.
But they cannot model and deal with temporal dynamics.
Although a class of neural ODE processes (NDPs) \citep{DBLP:conf/iclr/NorcliffeBDML21} generalized NPs defined over time by combining latent neural ODEs with NPs, it only aims at temporal dynamics and ignores spatial interactions between temporal dynamics on networks, resulting in insufficient behavior tracking for network dynamics learning.

\subsection{Neural ODE processes for network dynamics}
We introduce Neural ODE Processes for Network Dynamics (NDP4ND), a new family of stochastic processes parameterized by neural networks, simultaneously considering time dynamics and node interactions on topological structures, which learn to approximate random spatio-temporal trajectory functions on networks.
\textit{We remark that the NDP4ND can be seen as an extension of NDPs \citep{DBLP:conf/iclr/NorcliffeBDML21} for network dynamics, taking into account interactions of time dynamics on networks.}

We establish a general stochastic process ($\mathcal{F}$) to model the complex trajectory function space, while considering the correlation between the states of adjacent nodes.
Specifically, we model a collection $\{X_l(t)\}$ as a stochastic process, where the node state $X_l(t)\in\mathbb{R}^d$ is a random variable (Extended Data Fig.~\ref{fig_skeleton}b).
In principle, as time and space (network) expand, there are infinitely many such random variables in the networked system.
To build the stochastic processes, the critical issue we face is how to construct a family of finite-dimensional distributions to completely characterize the statistical relationship among any finite set of these random variables.

\textbf{Family of finite-dimensional distributions}.
Since the NDP4ND should model the dynamic changes of node states on the networked systems, we properly combine the universal skeleton of network dynamics \citep{2013Universality}, latent neural ODE \citep{NEURIPS2018_69386f6b}, and neural processes \citep{DBLP:journals/corr/abs-1807-01622} to build its family of distributions suitable for network dynamics modeling, retaining advantages including broad network interactions, temporal changes, and data-adaptation.
Extended Data Fig.~\ref{fig_skeleton}c shows the probabilistic graphical model of the generative process behind the NDP4ND.

Given a network $A$ and a context set  $\mathcal{D}_{\mathbb{C}}$,
the generative model can be formulated as
\begin{equation}
 p_{X,z,L(0)}=
   {p(z\vert \mathcal{D}_{\mathbb{C}},A)}
    {p(L(0)\vert \mathcal{D}_{\mathbb{C}})}
    \prod_{i=1}^{N_{\mathbb{T}}}{p(X_{l^{\mathbb{T}}_{i}}(t^{\mathbb{T}}_{i})\vert t^{\mathbb{T}}_i,l^{\mathbb{T}}_{i},z,L(0),A)},
\label{eq1}
\end{equation}
where $p_{X,z,L(0)}$ is a shorthand for the generative model $
 p(X_{l^{\mathbb{T}}_{1:N_{\mathbb{T}}}}(t^{\mathbb{T}}_{1:N_{\mathbb{T}}}),z,L(0)\vert t^{\mathbb{T}}_{1:N_{\mathbb{T}}},l^{\mathbb{T}}_{1:N_{\mathbb{T}}},\mathcal{D}_{\mathbb{C}},A)
$,
$t^{\mathbb{T}}_{1:N_{\mathbb{T}}}$, $l^{\mathbb{T}}_{1:N_{\mathbb{T}}}$, and $X_{l^{\mathbb{T}}_{1:N_{\mathbb{T}}}}(t^{\mathbb{T}}_{1:N_{\mathbb{T}}})$ are the target set, $L(0)$ and $z$ denote the initial states in latent space and the global random vector that can control the network ODE, respectively.

\textbf{Modeling distributions ${p(z\vert \mathcal{D}_{\mathbb{C}},A)}$ and ${p(L(0)\vert \mathcal{D}_{\mathbb{C}})}$}.
We encode network topology and context set into two latent variables, i.e. $z\sim q(z\vert \mathcal{D}_{\mathbb{C}},A)$ and $L(0)\sim q(L(0)\vert \mathcal{D}_{\mathbb{C}})$.
To parameterize the distribution of $z$, we first integrate observations with the same time in the contexts and combine them with the topology to form a graph, i.e. $G(t)=(X(t), M(t), A)$, where $X(t)\in\mathbb{R}^{n\times d}$ stores the all observed states at time $t$ and the mask $M(t)$ has the same shape as $X(t)$, indicating which nodes have been observed (its value on the corresponding position is set to 1, otherwise it is 0).
The context set $\mathcal{D}_{\mathbb{C}}$ is then transformed into a graph set $\{(t_1,G(t_1)),...,(t_K,G(t_K))\}$, where $K$ is consistent with the number of $t$ in contexts that eliminate duplicates.
We use a neural network $\bm{\varphi}$ to produce a representation for each pair in the graph set, i.e., $r_k = \bm{\varphi}(t_k,G(t_k))$, and then use another neural network $\bm{\rho}$ to give the distribution of $z$ by $q(z\vert \mathcal{D}_{\mathbb{C}},A)=\mathcal{N}(\mu_z,diag(\sigma_z^2))$, where $[\mu_z;\sigma_z^2]=\bm{\rho}(r)$, $r=agg(\{r_k\}_{k=1}^{K})$, and $agg$ is an aggregation operation with permutation invariance, such as element-wise mean.
Although the distribution of $L(0)$ can also be obtained through the above fashion, we handle it in a simpler way when the initial states are always known, i.e., $q(L(0)\vert \mathcal{D}_{\mathbb{C}})=q(L(0)\vert X(0))=\prod_{l=1}^{n}
\mathcal{N}(\mu_{L_l(0)},diag(\sigma_{L_l(0)}^2))$, where $[\mu_{L_l(0)};\sigma_{L_l(0)}^2]=\bm{e}(X_l(0))$ and function $\bm{e}$ is as a neural network for encoding each initial state at any node $l$.

\textbf{Modeling distribution ${p(X_{l^{\mathbb{T}}_{i}}(t^{\mathbb{T}}_{i})\vert t^{\mathbb{T}}_i,l^{\mathbb{T}}_{i},z,L(0),A)}$}.
After obtaining the distributions of $z$ and $L(0)$, we perform dynamic propagation in the hidden space with initial hidden state $L(0)$ based on the universal skeleton of network dynamics as follows
\begin{equation}
    L_{l^{\mathbb{T}}_{i}}(t^{\mathbb{T}}_i)=L_{l^{\mathbb{T}}_{i}}(0)+\int_{t_0}^{t^{\mathbb{T}}_i}{\bigg( \bm{S}(L_{l^{\mathbb{T}}_i}(t),\tilde{z})+\sum_{j=1}^{n}A_{l^{\mathbb{T}}_i,l^{\mathbb{T}}_j}\bm{I}(L_{l^{\mathbb{T}}_i}(t),L_{l^{\mathbb{T}}_j}(t),\tilde{z})}\bigg)dt,
\end{equation}
where $\bm{S}$ and $\bm{I}$ are neural networks, representing self dynamics and interaction dynamics, respectively.
Note that the dynamic process is mainly controlled by $\tilde{z}=[z;r]$, where $r$ transmits the context information to the propagation process from a deterministic path, and random variables $z$ and $L(0)$ jointly characterize the uncertainty of underlying network dynamics.
Given $\tilde{z}$ and $L_{l^{\mathbb{T}}_{i}}(0)$, the evolved latent state $L_{l^{\mathbb{T}}_{i}}(t^{\mathbb{T}}_i)$ can be seen as a deterministic function dominated by $t^{\mathbb{T}}_i$ and $l^{\mathbb{T}}_{i}$, involving the interaction between nodes.
Assuming that output states are noisy, for a given $L_{l^{\mathbb{T}}_{i}}(t^{\mathbb{T}}_i)$, we can decode it into the predictive state by
$X_{l^{\mathbb{T}}_{i}}(t^{\mathbb{T}}_i)\sim p(X_{l^{\mathbb{T}}_{i}}(t^{\mathbb{T}}_i)|t^{\mathbb{T}}_i,l^{\mathbb{T}}_{i},z,L(0),A)=\mathcal{N}(\mu_{{X}_{l^{\mathbb{T}}_i}}(t^{\mathbb{T}}_i),\sigma^2_{{X}_{l^{\mathbb{T}}_i}}(t^{\mathbb{T}}_i))$,
where $[\mu_{{X}_{l^{\mathbb{T}}_i}}(t^{\mathbb{T}}_i); \sigma^2_{{X}_{l^{\mathbb{T}}_i}}(t^{\mathbb{T}}_i)]=\bm{d}(t^{\mathbb{T}}_i,L_{l^{\mathbb{T}}_{i}}(t^{\mathbb{T}}_i),\tilde{z})$ 
and $\bm{d}$ is a neural network.

\textbf{Theoretical existence.}
We state that the constructed family of finite-dimensional distributions satisfies the exchangeability and consistency conditions (Proposition 1 in the Section A.2 of the supplementary information).
The Kolmogorov Extension Theorem guarantees that these conditions are sufficient to define a stochastic process \citep{oksendal2003}. 
In other words, the stochastic process we established exist, and the family of distributions is just its.
The specific architectures of the neural networks in the NPD4NP ($\bm{\varphi}$, $\bm{\rho}$, $\bm{e}$, $\bm{S}$, $\bm{I}$, and $\bm{d}$) can be found in the Section A.1.1 of the supplementary information.

\textbf{Training.}
To learn the distribution over trajectory functions, we train the model on a set of state observations generated from $B$ different network dynamics with $A^{(1)},...,A^{(B)}$ topological structures.
Also, we split the observations into $B$ context sets $\mathcal{D}^{1:B}_{\mathbb{C}}$ and $B$ target sets $\mathcal{D}^{1:B}_{\mathbb{T}}$. 
$\mathcal{D}^{b}_{\mathbb{C}}$ is usually a subset of $\mathcal{D}^{b}_{\mathbb{T}}$.
Since the generative process (Eq.~\ref{eq1}) of the NDP4ND contains multiple highly non-linear neural networks, the true posterior is intractable. 
We, thus, follow \citep{pmlr-v80-garnelo18a,DBLP:journals/corr/abs-1807-01622,DBLP:conf/iclr/NorcliffeBDML21} and use an amortized variational inference to train the parameters by reparametrization trick and stochastic gradient descent.
The derived total training loss under all observations is as follows
\begin{equation}
\small
        \mathcal{L}=\frac{1}{B}{\sum_{b=1}^{B}{\bigg[{\mathbb{E}_{q_{z,L(0)}}\bigg(\sum_{i=1}^{N^{b}_{\mathbb{T}}}{-\log{p(X_{l_{i}^{\mathbb{T}_{b}}}(t_{i}^{\mathbb{T}_{b}})\vert t_{i}^{\mathbb{T}_{b}},l_{i}^{\mathbb{T}_{b}},z,L(0),A^{(b)})}}\bigg)+\beta{KL}(q(z|\mathcal{D}^{b}_{\mathbb{T}})\Vert q(z|\mathcal{D}^{b}_{\mathbb{C}}))}\bigg]}},
\label{eq_loss}
\end{equation}
where $q_{z,L(0)}=q(z|\mathcal{D}^{b}_{\mathbb{T}},A^{(b)})q(L(0)|\mathcal{D}^{b}_{\mathbb{T}})$ and $\beta$ is a tunable parameter to relieve the KL vanishing.

\textbf{Prediction.}
Given the learned NDP4ND, a network topology $A$, and a set of state observations $\mathcal{D}_{\mathbb{C}}
$ from the task at hand,
we obtain the predictive distribution of the states on any node $l^{\mathbb{T}}$ at any time $t^{\mathbb{T}}$ as
\begin{equation}
    p(X_{l^{\mathbb{T}}}(t^{\mathbb{T}})\vert t^{\mathbb{T}},l^{\mathbb{T}},\mathcal{D}_{\mathbb{C}},A)=\int
q(z|\mathcal{D}_{\mathbb{C}},A)q(L(0)|\mathcal{D}_{\mathbb{C}})p(X_{l^{\mathbb{T}}}(t^{\mathbb{T}})\vert t^{\mathbb{T}},l^{\mathbb{T}},z,L(0),A)dzdL(0).
\label{eq_pred}
\end{equation}
We can use the Monte Carlo method to approximate the integral and moment matching to construct a Gaussian posterior approximation for the distribution as
$
p(X_{l^{\mathbb{T}}}(t^{\mathbb{T}})\vert t^{\mathbb{T}},l^{\mathbb{T}},\mathcal{D}_{\mathbb{C}},A)\approx \mathcal{N}(\mathcal{Z}_1,\mathcal{Z}_2)
$,
where 
$\mathcal{Z}_1=\frac{1}{J}\sum_{j=1}^{J}{\mu^{(j)}_{X_{l^{\mathbb{T}}}}(t^{\mathbb{T}})}$,
$\mathcal{Z}_2=[\frac{1}{J}\sum_{j=1}^{J}{(\sigma^{(j)}_{X_{l^{\mathbb{T}}}}(t^{\mathbb{T}}))^2+(\mu^{(j)}_{X_{l^{\mathbb{T}}}}(t^{\mathbb{T}}))^2}]-[\frac{1}{J}\sum_{j=1}^{J}{\mu^{(j)}_{X_{l^{\mathbb{T}}}}(t^{\mathbb{T}})}]^2$,
$J$ is the sampling number for $z$ and $L(0)$,
and $\mu^{(j)}_{X_{l^{\mathbb{T}}}}(t^{\mathbb{T}})$ and $\sigma^{(j)}_{X_{l^{\mathbb{T}}}}(t^{\mathbb{T}})$ are outputs of decoder network $\bm{d}$.
The full derivation of the training loss and predictions can be found in Sections A.1.2 and A.1.3 of the supplementary information.

\backmatter

\bmhead{Acknowledgments}
This work was supported by the National Key R\&D Program of China under Grant Nos. 2021ZD0112501 and 2021ZD0112502;  the National Natural Science Foundation of China under Grant Nos. U22A2098, 62172185, 62206105 and 62202200.

\bmhead{Author contributions}
J.C. and B.Y. conceived of the presented idea and designed the experiments.
J.C. and B.S. carried out the experiments.
B.S. and J.L. analyzed the data.
J.C. wrote the paper.

\bmhead{Competing interests}
The authors declare no competing interests

\bmhead{Supplementary information}
The online version contains supplementary information available at \url{xxx}.

\bmhead{Data and code availability}
The data and source code are freely available at GitHub (\url{https://github.com/csjtx1021/neural_ode_processes_for_network_dynamics-master}) to ensure the reproduction of our experimental results.

\bibliography{ms}

\clearpage

\renewcommand{\figurename}{Extended Data Fig.}
\setcounter{figure}{0}

\begin{figure}[htb]%
\centering
\includegraphics[width=1.\textwidth]{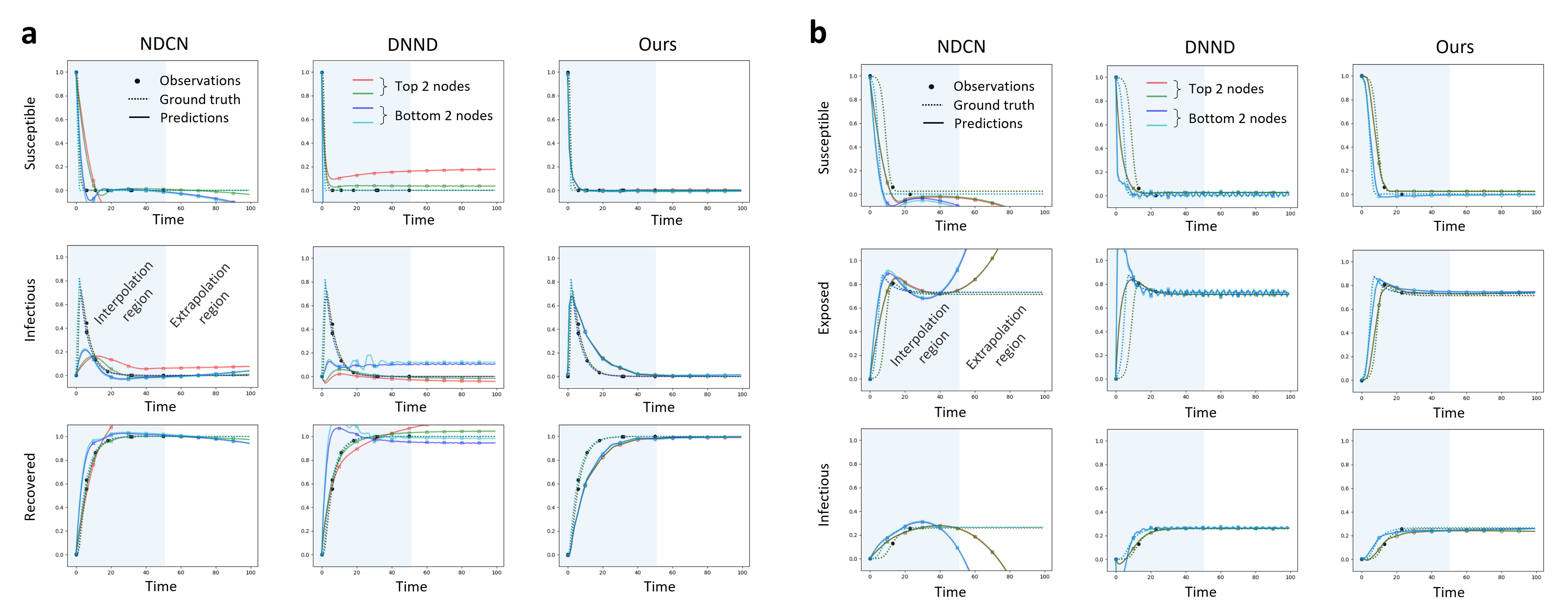}
\caption{
The interpolation and extrapolation results on \textit{susceptible-infectious-recovered (SIR)} and \textit{susceptible-exposed-infectious-susceptible (SEIS)}.
\textbf{a}, The testing results of the NDCN, DNND, and our NDP4ND on the SIR.
The ratio of observations in this testing network ODE is $6.89\%$.
\textbf{b}, The testing results of the NDCN, DNND, and our NDP4ND on the SEIS.
The ratio of observations in this testing network ODE is $4.08\%$.
The number of nodes and the maximum value of all observed times ($T_o$) are 200 and 50, respectively.
Overall, our NDP4ND can learn the most effective network dynamics from irregularly-sampled, partial, and sparse observations.
}\label{fig_sir_seis}
\end{figure}

\renewcommand{\tablename}{Extended Data Table}
\setcounter{table}{0}

\begin{table}[htb]
\centering
\caption{The interpolation and extrapolation results across all testing network ODEs on \textit{Mutualistic Interaction} dynamics in terms of quantitative predictive error and similarity.
The reported values are the \textit{mean} (\textit{standard deviation}) of the Mean Absolute Error (MAE) and Dynamic Time Warping (DTW) between prediction and ground truth.
The best results are bolded.
}
\label{table_mutu}
\begin{tabular}{ccrr}
\toprule
Methods                       & \multicolumn{1}{c}{Metrics} & \multicolumn{1}{c}{Interpolation ($T_o\leq50$)} & \multicolumn{1}{c}{Extrapolation ($T_o>50$)} \\ \midrule
\multirow{2}{*}{LG-ODE}       & MAE                  & 4.49E+02 (2.24E+03)                         & 1.33E+03 (6.66E+03)                         \\
                              & DTW                  & 2.24E+04 (1.12E+05)                          & 6.67E+04 (3.33E+05)            \\\midrule
\multirow{2}{*}{NDCN}         & MAE                  & 5.08E+00 (1.34E+01)                         & 1.98E+02 (1.50E+03)                         \\
                              & DTW                  & 2.34E+02 (6.69E+02)                     & 9.89E+03 (7.50E+04)                         \\\midrule
\multirow{2}{*}{DNND}         & MAE                  & 7.40E-01 (8.96E-01)                         & 1.57E+03 (1.57E+04)                         \\
                              & DTW                  & 2.69E+01 (4.59E+01)                         & 7.87E+04 (7.83E+05)                    \\ \midrule\midrule
\multirow{2}{*}{NDP4ND(Ours)} & MAE                  & \textbf{3.62E-01 (2.16E-01)}                & \textbf{3.67E-01 (2.73E-01)}                     \\
                              & DTW                  & \textbf{1.43E+01 (1.01E+01)}                & \textbf{1.82E+01 (1.37E+01)}              \\ 
\bottomrule
\end{tabular}
\end{table}

\begin{table}[htb]
\centering
\caption{The interpolation and extrapolation results across all testing network ODEs on \textit{Second-order Phototaxis} dynamics in terms of quantitative predictive error and similarity.
The reported values are the \textit{mean} (\textit{standard deviation}) of the Mean Absolute Error (MAE) and Dynamic Time Warping (DTW) between prediction and ground truth.
The best results are bolded.
}
\label{table_phototaxis}
\begin{tabular}{ccrr}
\toprule
Methods                       & \multicolumn{1}{c}{Metrics} & \multicolumn{1}{c}{Interpolation ($T_o\leq0.5$)} & \multicolumn{1}{c}{Extrapolation ($T_o>0.5$)} \\ \midrule
\multirow{2}{*}{NDCN}         & MAE                  & 5.78E+00 (1.67E+00)                         & 1.31E+01 (9.77E+00)                         \\
                              & DTW                  & 2.49E+02 (8.36E+01)                     & 6.31E+02 (4.97E+02)                         \\\midrule
\multirow{2}{*}{DNND}         & MAE                  & 9.92E-01 (2.65E-01)                         & 2.85E+00 (1.17E+00)                         \\
                              & DTW                  & 1.74E+01 (1.22E+01)                         & 1.18E+02 (5.53E+01)           \\ \midrule\midrule
\multirow{2}{*}{NDP4ND(Ours)} & MAE                  & \textbf{6.00E-01 (1.18E-01)}                & \textbf{9.09E-01 (4.87E-01)}                     \\
                              & DTW                  & \textbf{1.48E+01 (2.89E+00)}                & \textbf{2.87E+01 (1.61E+01)}              \\ 
\bottomrule
\end{tabular}
\end{table}

\begin{table}[htb]
\centering
\caption{The interpolation and extrapolation results across all testing network ODEs on \textit{Brain} dynamics in terms of quantitative predictive error and similarity.
The reported values are the \textit{mean} (\textit{standard deviation}) of the Mean Absolute Error (MAE) and Dynamic Time Warping (DTW) between prediction and ground truth.
The best results are bolded.
}
\label{table_brain}
\begin{tabular}{ccrr}
\toprule
Methods                       & \multicolumn{1}{c}{Metrics} & \multicolumn{1}{c}{Interpolation ($T_o\leq0.5$)} & \multicolumn{1}{c}{Extrapolation ($T_o>0.5$)} \\ \midrule
\multirow{2}{*}{NDCN}         & MAE                  & 1.79E+00 (8.85E-01)                         & 2.65E+01 (2.32E+01)                         \\
                              & DTW                  & 7.30E+01 (4.14E+01)                     & 1.32E+03 (1.16E+03)                         \\\midrule
\multirow{2}{*}{DNND}         & MAE                  & 1.03E+00 (1.21E-01)                         & 1.48E+00 (2.35E-01)                         \\
                              & DTW                  & 4.01E+01 (1.35E+01)                         & 5.79E+01 (6.98E+00)           \\ \midrule\midrule
\multirow{2}{*}{NDP4ND(Ours)} & MAE                  & \textbf{2.91E-01 (3.64E-02)}                & \textbf{6.06E-01 (1.02E-01)}                     \\
                              & DTW                  & \textbf{6.66E+00 (1.08E+00)}                & \textbf{1.46E+01 (1.54E+00)}              \\ 
\bottomrule
\end{tabular}
\end{table}

\begin{table}[htb]
\centering
\caption{The interpolation and extrapolation results across all testing network ODEs on \textit{Compartment Models} in epidemiology in terms of quantitative predictive error and similarity.
The reported values are the \textit{mean} (\textit{standard deviation}) of the Mean Absolute Error (MAE) and Dynamic Time Warping (DTW) between prediction and ground truth.
The best results are bolded.
}
\label{table_compartment}
\begin{tabular}{cccrr}
\toprule
Dynamics    & Methods                       & \multicolumn{1}{l}{Metrics} & \multicolumn{1}{c}{Interpolation ($T_o\leq50$)} & \multicolumn{1}{c}{Extrapolation ($T_o>50$)} \\ \midrule
\multirow{6}{*}{SIS}  & \multirow{2}{*}{NDCN}         & MAE & \textbf{6.32E-02 (3.94E-02)}   & 4.95E-01 (6.48E-01)  \\
                      &                               & DTW & \textbf{1.81E+00 (1.40E+00)}   & 2.48E+01 (3.24E+01)  \\\cmidrule{2-5} 
                      & \multirow{2}{*}{DNND}         & MAE & 1.27E-01 (1.79E-01)   & 1.13E+01 (4.88E+01) \\
                      &                               & DTW & 4.03E+00 (7.96E+00)   & 5.64E+02 (2.44E+03) \\ \cmidrule{2-5}  
                      & \multirow{2}{*}{NDP4ND(Ours)} & MAE & 1.04E-01 (7.91E-02)   & \textbf{4.45E-02 (4.34E-02)}   \\
                      &                               & DTW & 1.93E+00 (1.88E+00)   & \textbf{2.20E+00 (2.19E+00)} \\\midrule\midrule
\multirow{6}{*}{SIR}  & \multirow{2}{*}{NDCN}         & MAE & \textbf{1.27E-01 (1.79E-01)}   & 3.32E-01 (4.26E-01)    \\
                      &                               & DTW & 5.28E+00 (9.03E+00)   & 1.66E+01 (2.13E+01) \\ \cmidrule{2-5} 
                      & \multirow{2}{*}{DNND}         & MAE & 5.62E+00 (2.41E+01)   & 1.64E+04 (7.14E+04)       \\
                      &                               & DTW & 2.81E+02 (1.21E+03)   & 8.18E+05 (3.57E+06)\\ \cmidrule{2-5} 
                      & \multirow{2}{*}{NDP4ND(Ours)} & MAE & 1.51E-01 (8.69E-02)   & \textbf{5.41E-02 (8.16E-02)}  \\
                      &                               & DTW & \textbf{3.39E+00 (3.69E+00)}   & \textbf{2.70E+00 (4.08E+00)} \\ \midrule\midrule
\multirow{6}{*}{SEIS} & \multirow{2}{*}{NDCN}         & MAE & \textbf{1.02E-01 (1.23E-01)}   & 9.42E-01 (1.25E+00)   \\
                      &                               & DTW & 3.49E+00 (6.18E+00)   & 4.71E+01 (6.24E+01) \\ \cmidrule{2-5} 
                      & \multirow{2}{*}{DNND}         & MAE & 2.12E-01 (2.71E-01)   & 3.54E-01 (6.51E-01)  \\
                      &                               & DTW & 7.21E+00 (1.37E+01)   & 1.77E+01 (3.26E+01)  \\ \cmidrule{2-5}
                      & \multirow{2}{*}{NDP4ND(Ours)} & MAE & 1.03E-01 (7.32E-02)   & \textbf{3.14E-02 (2.12E-02)}  \\
                      &                               & DTW & \textbf{1.27E+00 (8.87E-01)}   & \textbf{1.57E+00 (1.06E+00)}  \\ 
\bottomrule
\end{tabular}
\end{table}

\begin{table}[htb]
\centering
\caption{The results of neural processes-based methods across all testing network ODEs in terms of quantitative predictive error and similarity.
The reported values are the \textit{mean} (\textit{standard deviation}) of the Mean Absolute Error (MAE) and Dynamic Time Warping (DTW) between prediction and ground truth.
The best results are bolded.
}
\label{table_mutu_NP}
\begin{tabular}{ccrr}
\toprule
Methods                       & \multicolumn{1}{c}{Metrics} & \multicolumn{1}{c}{Interpolation ($T_o\leq50$)} & \multicolumn{1}{c}{Extrapolation ($T_o>50$)} \\ \midrule
\multirow{2}{*}{NP$^{1\dagger}$}           & MAE                  & 5.77E+00 (1.37E+00)                         & 5.79E+00 (1.45E+00)               \\
                              & DTW                  & 2.86E+02 (6.90E+01)                         & 2.89E+02 (7.24E+01)                 \\\midrule
\multirow{2}{*}{NDP$^{2\dagger}$}          & MAE                  & 5.56E+00 (1.42E+00)                         & 5.75E+00 (1.46E+00)                         \\
                              & DTW                  & 2.76E+02 (7.04E+01)                         & 2.87E+02 (7.30E+01)       \\ \midrule
\multirow{2}{*}{NDP4ND w/o ode$^3$}           & MAE                  & 7.95E-01 (4.23E-01)                         & 7.75E-01 (4.85E-01)               \\
                              & DTW                  & 3.17E+01 (1.95E+01)                         & 3.87E+01 (2.43E+01)                 \\\midrule
\multirow{2}{*}{NDP4ND w/o z$^4$}          & MAE                  & 1.09E+00 (8.86E-01)                         & 1.06E+00 (8.81E-01)                         \\
                              & DTW    & 4.96E+01 (4.45E+01)                         & 5.31E+01 (4.41E+01)      \\\midrule \midrule     
\multirow{2}{*}{NDP4ND(Ours)} & MAE                  & \textbf{3.62E-01 (2.16E-01)}                & \textbf{3.67E-01 (2.73E-01)}                     \\
                              & DTW                  & \textbf{1.43E+01 (1.01E+01)}                & \textbf{1.82E+01 (1.37E+01)}              \\ 
\bottomrule
\end{tabular}
\footnotetext[1]{Neural processes (NP) \citep{DBLP:journals/corr/abs-1807-01622}: Stochastic processes parameterized by neural networks.}
\footnotetext[2]{Neural ODE processes (NDP) \citep{DBLP:conf/iclr/NorcliffeBDML21}: A class of stochastic processes, generalized the NPs defined over time by combining latent neural ODEs with NPs.}
\footnotetext[3]{A variant of our NDP4ND, removing ODE flow in the architecture, i.e., deleting $L_{l_{i}^{\mathbb{T}}}(t_{i}^{\mathbb{T}})$ from the input of neural network $\bm{d}$. This is to explore the impact of the ODE process of the NDP4ND on the results.}
\footnotetext[4]{A variant of our NDP4ND, removing $\tilde{z}$ from the input of neural network $\bm{d}$ in the architecture. This is to explore the impact of skip connection (i.e., add an information flow path that skips the ODE process) on the results.}
\footnotetext[\dagger]{Note that the NP and NDP are not specifically designed for network dynamics learning.
Using them as comparison methods is to verify the necessity of introducing interactions of time dynamics on networks.
Due to their limited flexibility, they cannot handle networked systems with arbitrary sizes. 
We thus give them the power to process network dynamics by utilizing the graph neural network to flexibly encode the observations.}
\end{table}

\begin{figure}[t]%
\centering
\includegraphics[width=1.\textwidth]{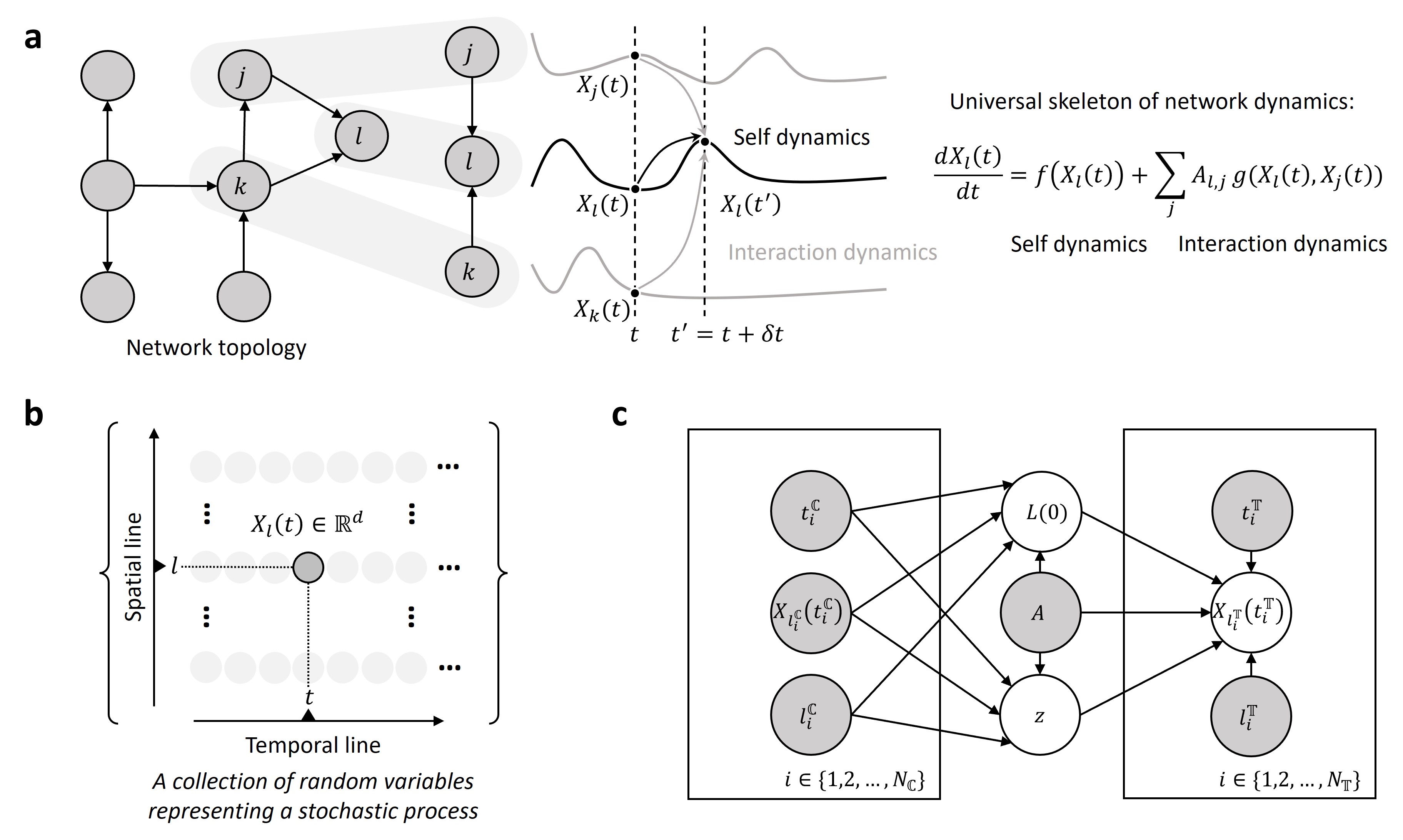}
\caption{
\textbf{a}, The state changes on node $l$ can be modeled by self dynamics and interaction dynamics, i.e., $\frac{dX_{l}(t)}{dt}=f(X_l(t))+\sum_{j}{A_{l,j}g(X_l(t),X_j(t))}$. 
The former quantifies the impact of one's own state on change, while the latter quantifies the impact of neighboring states on change.
\textbf{b}, We model a collection of $\{X_l(t)\}$ as a stochastic process, where the state $X_l(t)$ on node $l$ at time $t$ is a random variable.
In principle, as time and space (network) expand, there are infinitely many such random variables in a networked system.
\textbf{c}, The probabilistic graphical model of the generative process behind the NDP4ND.
$z$ and $L(0)$ denote the global random vector and the initial states in latent space respectively, through which uncertainty is introduced to enhance the model prediction facing scarce and noisy observations.
}\label{fig_skeleton}
\end{figure}

\end{document}


\title[Article Title]{\textit{Supplementary Information} of {Learning Continuous Network Emerging Dynamics from Scarce Observations via Data-Adaptive Stochastic Processes}}

\maketitle

\tableofcontents



\clearpage

\begin{appendix}

\section{Model Details}\label{sec1}

We propose the Neural ODE Processes for Network Dynamics (NDP4ND) to represent a spatio-temporal trajectory function space of network dynamics for covering abundant network dynamics instances.
We first introduce the details of the NDP4ND, including the complete description of its computational diagram, shown in Fig.~1 in main text, and the architecture design of neural networks.
The proofs of the propositions in the main text will be provided later.

\subsection{Details of the Neural ODE Processes for Network Dynamics}

The overall computational diagram consists of the following five main steps (corresponding to Fig.~1 in main text):
\begin{itemize}
\item[\textbf{a,}] Let the observations ($\mathcal{D}_{\mathbb{C}}$) that capture the empirical topological structure and spatio-temporal behaviors be as the input of the NDP4ND.
\item[\textbf{b,}] Observations $\mathcal{D}_{\mathbb{C}}$ are first integrated into graphs based on different timestamps $t$, i.e., $G(t)=(X(t), M(t), A)$.
$X(t)\in\mathbb{R}^{n\times d}$ stores all observed states at time $t$ and the mask $M(t)$ has the same shape as $X(t)$, indicating which nodes have been observed (its value on the corresponding position is set to 1, otherwise it is 0). 
Fig.~\ref{fig_g_t} provides an illustration for $G(t)$.
Then, by using a neural network $\bm{\varphi}$, the observations can be encoded into the embedding ($r$) with considering the temporal and topological information.
Based on the embedding ($r$) of observations, a neural network $\bm{\rho}$ can be used to give the distribution of the global latent vector, i.e., $q(z|\mathcal{D}_{\mathbb{C}},A)$. 
The distribution is to sample specific global control variables $z$ and introduces uncertainty for propagation equations of network dynamics in hidden spaces.
\begin{figure}[t]%
\centering
\includegraphics[width=.6\textwidth]{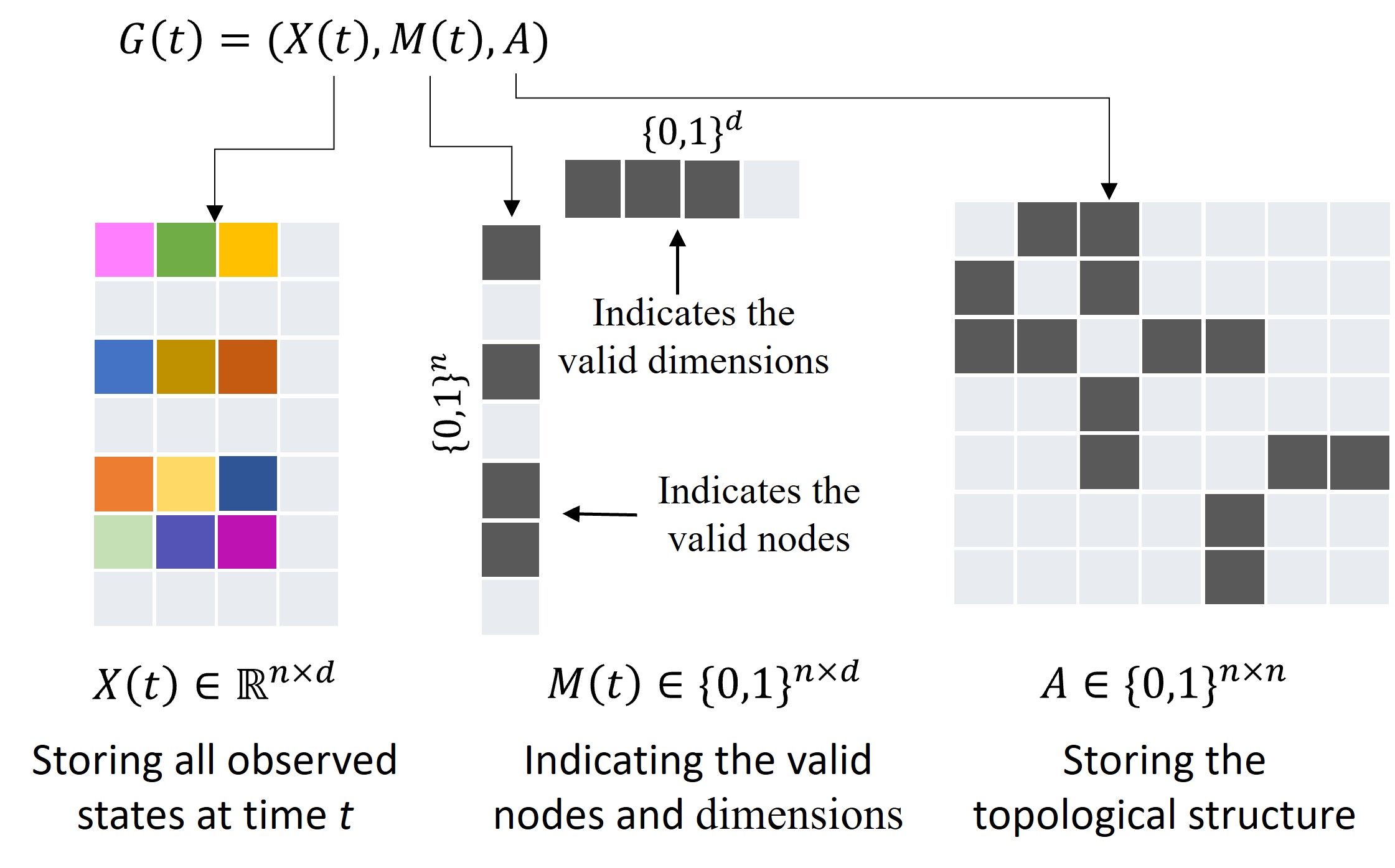}
\caption{
An illustration for $G(t)=(X(t),M(t),A)$.
$X(t)\in\mathbb{R}^{n\times d}$ stores all observed states at time $t$ and the mask $M(t)$ has the same shape as $X(t)$, indicating which nodes have been observed (its value on the corresponding position is set to 1, otherwise it is 0).
}\label{fig_g_t}
\end{figure}
\item[\textbf{c,}] Meanwhile, based on observations, the distribution of the initial latent state ($q(L(0)|\mathcal{D}_{\mathbb{C}},A)$) can be obtained through a neural network $\bm{e}$.
\item[\textbf{d,}] Built on sampled global control variable $z$ and deterministic representation of observations $r$, the neural dynamics equation in the hidden space is constructed via the universal skeleton of network dynamics\citep{2013Universality}, i.e.,
$
\dot{L}_{l^{\mathbb{T}}_{i}}(t)=\bm{S}(L_{l^{\mathbb{T}}_i}(t),\tilde{z})+\sum_{j=1}^{n}A_{l^{\mathbb{T}}_i,l^{\mathbb{T}}_j}\bm{I}(L_{l^{\mathbb{T}}_i}(t),L_{l^{\mathbb{T}}_j}(t),\tilde{z})
$,
where $\tilde{z}=[z;r]$ and $n$ is the system size.
$\bm{S}$ and $\bm{I}$ are neural networks that characterize self dynamics and interaction dynamics, respectively.
\item[\textbf{e,}] Given the neural dynamics equation, topology $A$, initial latent state $L(0)$, test node $l^{\mathbb{T}}_i$, and test time $t^{\mathbb{T}}_i$, we can evolve the hidden state by solving the network ordinary differential equation (ODE) as
$
    L_{l^{\mathbb{T}}_{i}}(t^{\mathbb{T}}_i)=L_{l^{\mathbb{T}}_{i}}(0)+\int_{t_0}^{t^{\mathbb{T}}_i}{\bigg( \bm{S}(L_{l^{\mathbb{T}}_i}(t),\tilde{z})+\sum_{j=1}^{n}A_{l^{\mathbb{T}}_i,l^{\mathbb{T}}_j}\bm{I}(L_{l^{\mathbb{T}}_i}(t),L_{l^{\mathbb{T}}_j}(t),\tilde{z})}\bigg)dt,
$.
Here we use the adjoint method \citep{NEURIPS2018_69386f6b} to solve it to ensure that the model can perform backpropagation to learn parameters.
Assuming that output states are noisy, for a given $L_{l^{\mathbb{T}}_{i}}(t^{\mathbb{T}}_i)$, we can decode it into the predictive state by
$X_{l^{\mathbb{T}}_{i}}(t^{\mathbb{T}}_i)\sim p(X_{l^{\mathbb{T}}_{i}}(t^{\mathbb{T}}_i)|t^{\mathbb{T}}_i,l^{\mathbb{T}}_{i},z,L(0),A)=\mathcal{N}(\mu_{{X}_{l^{\mathbb{T}}_i}}(t^{\mathbb{T}}_i),\sigma^2_{{X}_{l^{\mathbb{T}}_i}}(t^{\mathbb{T}}_i))$,
where $[\mu_{{X}_{l^{\mathbb{T}}_i}}(t^{\mathbb{T}}_i); \sigma^2_{{X}_{l^{\mathbb{T}}_i}}(t^{\mathbb{T}}_i)]=\bm{d}(t^{\mathbb{T}}_i,L_{l^{\mathbb{T}}_{i}}(t^{\mathbb{T}}_i),\tilde{z})$ 
and $\bm{d}$ is as a neural network.
Using $\tilde{z}$ as the input of $\bm{d}$ is equivalent to add an information flow path that skips the ODE process, i.e., a skip connection.
This makes the NDP4ND at least as expressive as Neural Processes (NPs) \citep{DBLP:journals/corr/abs-1807-01622} and Neural ODE Processes (NDPs) \citep{DBLP:conf/iclr/NorcliffeBDML21} in principle, also enabling it to model network dynamics that are not solely determined by some underlying ODEs.

\end{itemize}

\subsubsection{Design of the neural networks in the NDP4ND}
The neural networks involved in the model architecture include: $\bm{\varphi}$, $\bm{\rho}$, $\bm{e}$, $\bm{S}$, $\bm{I}$, and $\bm{d}$.
We provide their specific implementations in the experiments below.

$\bm{\varphi}: [t_k;G(t_k)]\rightarrow r_k$.
\begin{itemize}[itemindent=1cm] 
            \item $h_t=enc_t(t_k)$, where $enc_t$ is a multilayer perceptron with 1 hidden layer and LeakyReLU activations.
            \item $h_{G(t_k)}=enc_G(G(t_k))$, $enc_G$ has 2 graph attention (GAT) layers with the head number of 8 and LeakyReLU activations. Note that the original states $X(t)$ are encoded before feeding $enc_G$, i.e., $h_{X(t)}=enc_X(X(t))$, where $enc_X$ is a multilayer perceptron with 1 hidden layer and LeakyReLU activations.
            \item $r_k=enc_r([h_t;h_{G(t)}])$, where $enc_r$ is a multilayer perceptron with 2 hidden layers and LeakyReLU activations.
\end{itemize}

To obtain a representation of the observations, we use an aggregation operation with permutation invariance to aggregate all $r_k$, i.e., $r=agg(\{r_k\}_{k=1}^{K})$. 
We choose the element-wise mean as the aggregation operation.

$\bm{\rho}: r\rightarrow [\mu_z; \sigma^2_z]$, where $\mu_z$ and $\sigma^2_z$ are the mean and variance of the distribution $q(z|\mathcal{D}_{\mathbb{C}},A)$.
\begin{itemize}[itemindent=1cm] 
    \item $h_z=enc_z(r)$, where $enc_z$ is a multilayer perceptron with 1 hidden layer and LeakyReLU activations.
    \item $\mu_z=enc_{\mu_z}(h_z)$, where $enc_{\mu_z}$ is a multilayer perceptron with 1 hidden layer.
    \item $\sigma_z=0.1 + 0.9 \times \text{sigmoid}{(enc_{\sigma_z}(h_z))}$, where $enc_{\sigma_z}$ is a multilayer perceptron with 1 hidden layer.
\end{itemize}

$\bm{e}: \mathcal{D}_{\mathbb{C}}\rightarrow [\mu_{L(0)},\sigma^2_{L(0)}]$, where $\mu_{L(0)}$ and $\sigma^2_{L(0)}$ are the mean and variance of the distribution $q(L(0)|\mathcal{D}_{\mathbb{C}},A)$. 
Although the distribution of $L(0)$ can also be obtained through the above fashion, we handle it in a simpler way when the initial states are always known.
\begin{itemize}[itemindent=1cm] 
    \item $h_{L(0)}=enc_{L(0)}(X(0))$, where $X(0)$ is a shorthand for the initial states of all nodes and $enc_{L(0)}$ is a multilayer perceptron with 1 hidden layer and LeakyReLU activations.
    \item $\mu_{L(0)}=enc_{\mu_{L(0)}}(h_z)$, where $enc_{\mu_{L(0)}}$ is a multilayer perceptron with 1 hidden layer.
    \item $\sigma_{L(0)}=0.1 + 0.9 \times \text{sigmoid}{(enc_{\sigma_{L(0)}}(h_{L(0)}))}$, where $enc_{\sigma_{L(0)}}$ is a multilayer perceptron with 1 hidden layer.
\end{itemize}

$\bm{S}: [L_{l^{\mathbb{T}}_i}(t);\tilde{z}]\rightarrow L^{self}_{l^{\mathbb{T}}_i}(t)$, where $\tilde{z}=[r;z]$ and $l^{\mathbb{T}}_i$ is the location of test node.
\begin{itemize}[itemindent=1cm] 
    \item $L^{self}_{l^{\mathbb{T}}_i}(t)=S(L_{l^{\mathbb{T}}_i}(t),\tilde{z})$, where $S$ is a multilayer perceptron with 2 hidden layers and LeakyReLU activations.
\end{itemize}

$\bm{I}: [L_{l^{\mathbb{T}}_i}(t);L_{l^{\mathbb{T}}_j}(t);\tilde{z}]\rightarrow L^{inter}_{l^{\mathbb{T}}_i,l^{\mathbb{T}}_j}(t)$, where $\tilde{z}=[r;z]$, $l^{\mathbb{T}}_i$ is the location of test node and $l^{\mathbb{T}}_j$ is the node that interacts with $l^{\mathbb{T}}_i$.
\begin{itemize}[itemindent=1cm] 
    \item $L^{inter}_{l^{\mathbb{T}}_i,l^{\mathbb{T}}_j}(t)=I(L_{l^{\mathbb{T}}_i}(t),L_{l^{\mathbb{T}}_j}(t),\tilde{z})$, where $I$ is a multilayer perceptron with 2 hidden layers and LeakyReLU activations.
\end{itemize}

$\bm{d}: [L_{l^{\mathbb{T}}_{i}}(t^{\mathbb{T}}_i);t^{\mathbb{T}}_i;\tilde{z}]\rightarrow [\mu_{{X}_{l^{\mathbb{T}}_i}}(t^{\mathbb{T}}_i);\sigma^2_{{X}_{l^{\mathbb{T}}_i}}(t^{\mathbb{T}}_i)]$, where $\tilde{z}=[r;z]$, $l^{\mathbb{T}}_i$ and $t^{\mathbb{T}}_i$ are the test node and test time, respectively. 
\begin{itemize}[itemindent=1cm] 
    \item $L^{dec}_{l^{\mathbb{T}}_{i}}(t^{\mathbb{T}}_i)=dec_{L}(L_{l^{\mathbb{T}}_{i}}(t^{\mathbb{T}}_i),t^{\mathbb{T}}_i,\tilde{z})$, where $dec_{L}$ a multilayer perceptron with 2 hidden layers and LeakyReLU activations.
    \item $\mu_{X_{l^{\mathbb{T}}_{i}}(t^{\mathbb{T}}_i)}=dec_{\mu_{X}}(L^{dec}_{l^{\mathbb{T}}_{i}}(t^{\mathbb{T}}_i))$, where $enc_{\mu_{X}}$ is a multilayer perceptron with 1 hidden layer.
    \item $\sigma_{X_{l^{\mathbb{T}}_{i}}(t^{\mathbb{T}}_i)}=0.01 + 0.99 \times \text{softplus}{(dec_{\sigma_{X}}(L^{dec}_{l^{\mathbb{T}}_{i}}(t^{\mathbb{T}}_i)))}$, where, $\text{softplus}(a)=\log{1+e^{a}}$ and $dec_{\sigma_{X}}$ is a multilayer perceptron with 1 hidden layer.
\end{itemize}

Although automatic machine learning technology, such as Bayesian optimization \citep{20183068}, can be used for fine-grained hyperparameter optimization, we empirically set the hidden dimension of all multilayer perceptrons to 20, and the hidden dimension of graph attention layers to 16.

\subsubsection{The derivation of training loss}
Given a network $A$ and a context set  $\mathcal{D}_{\mathbb{C}}$,
we revisit the generative model of the NDP4ND as
\begin{equation*}
 p_{X,z,L(0)}=
   {p(z\vert \mathcal{D}_{\mathbb{C}},A)}
    {p(L(0)\vert \mathcal{D}_{\mathbb{C}},A)}
    \prod_{i=1}^{N_{\mathbb{T}}}{p(X_{l^{\mathbb{T}}_{i}}(t^{\mathbb{T}}_{i})\vert t^{\mathbb{T}}_i,l^{\mathbb{T}}_{i},z,L(0),A)},
\label{eq1}
\end{equation*}
where $p_{X,z,L(0)}$ is a shorthand for the generative model $
 p(X_{l^{\mathbb{T}}_{1:N_{\mathbb{T}}}}(t^{\mathbb{T}}_{1:N_{\mathbb{T}}}),z,L(0)\vert t^{\mathbb{T}}_{1:N_{\mathbb{T}}},l^{\mathbb{T}}_{1:N_{\mathbb{T}}},\mathcal{D}_{\mathbb{C}},A)
$,
$t^{\mathbb{T}}_{1:N_{\mathbb{T}}}$, $l^{\mathbb{T}}_{1:N_{\mathbb{T}}}$, and $X_{l^{\mathbb{T}}_{1:N_{\mathbb{T}}}}(t^{\mathbb{T}}_{1:N_{\mathbb{T}}})$ are the target set, $L(0)$ and $z$ denote the initial states in latent space and the global random vector that can control the network ODE, respectively.

Due to the embedded neural networks, the generative process is highly nonlinear, resulting in intractability to calculate the true posterior.
We, thus, follow \citep{pmlr-v80-garnelo18a,DBLP:journals/corr/abs-1807-01622,DBLP:conf/iclr/NorcliffeBDML21} and use an amortized variational inference to learn the parameters in the neural networks.

Given a network $A$, a context set 
$\mathcal{D}_{\mathbb{C}}=
\{(t^{\mathbb{C}}_1,l^{\mathbb{C}}_1,X_{l^{\mathbb{C}}_1}(t^{\mathbb{C}}_1)),...,(t^{\mathbb{C}}_{N_{\mathbb{C}}},l^{\mathbb{C}}_{N_{\mathbb{C}}},X_{l^{\mathbb{C}}_{N_{\mathbb{C}}}}(t^{\mathbb{C}}_{N_{\mathbb{C}}}))\}
$ and a target set $\mathcal{D}_{\mathbb{T}}=
\{(t^{\mathbb{T}}_1,l^{\mathbb{T}}_1,X_{l^{\mathbb{T}}_1}(t^{\mathbb{T}}_1)),...,(t^{\mathbb{T}}_{N_{\mathbb{T}}},l^{\mathbb{T}}_{N_{\mathbb{T}}},X_{l^{\mathbb{T}}_{N_{\mathbb{T}}}}(t^{\mathbb{T}}_{N_{\mathbb{T}}}))\}
$, the derivation process of the variational evidence lower-bound (ELBO) is as follows
\begin{equation*}
\centering
\begin{aligned}
&\log{p(X_{l^{\mathbb{T}}_{1:N_{\mathbb{T}}}}(t^{\mathbb{T}}_{1:N_{\mathbb{T}}})\vert t^{\mathbb{T}}_{1:N_{\mathbb{T}}},l^{\mathbb{T}}_{1:N_{\mathbb{T}}},\mathcal{D}_{\mathbb{C}},A)}\\
=&\log{\frac{{p(z\vert \mathcal{D}_{\mathbb{C}},A)}
    {p(L(0)\vert \mathcal{D}_{\mathbb{C}},A)}
    \prod_{i=1}^{N_{\mathbb{T}}}{p(X_{l^{\mathbb{T}}_{i}}(t^{\mathbb{T}}_{i})\vert t^{\mathbb{T}}_i,l^{\mathbb{T}}_{i},z,L(0),A)}}{p(z|\mathcal{D}_{\mathbb{T}},A)p(L(0)|\mathcal{D}_{\mathbb{T}},A)}}\\
=&\mathbb{E}_{q(z|\mathcal{D}_{\mathbb{T}},A)q(L(0)|\mathcal{D}_{\mathbb{T}},A)}\Big[\sum_{i=1}^{N_{\mathbb{T}}}{\log{p(X_{l_{i}^{\mathbb{T}}}(t_{i}^{\mathbb{T}})\vert t_{i}^{\mathbb{T}},l_{i}^{\mathbb{T}},z,L(0),A)}} \Big]\\
&-{KL}(q(z|\mathcal{D}_{\mathbb{T}},A)\Vert q(z|\mathcal{D}_{\mathbb{C}},A))-{KL}(q(L(0)|\mathcal{D}_{\mathbb{T}},A)\Vert q(L(0)|\mathcal{D}_{\mathbb{C}},A))\\
&+{KL}(q(z|\mathcal{D}_{\mathbb{T}},A)\Vert q(z|\mathcal{D}_{\mathbb{T}},A))+{KL}(q(L(0)|\mathcal{D}_{\mathbb{T}},A)\Vert q(L(0)|\mathcal{T}_{\mathbb{C}},A))\\
\geq&\mathbb{E}_{q(z|\mathcal{D}_{\mathbb{T}},A)q(L(0)|\mathcal{D}_{\mathbb{T}},A)}\Big[\sum_{i=1}^{N_{\mathbb{T}}}{\log{p(X_{l_{i}^{\mathbb{T}}}(t_{i}^{\mathbb{T}})\vert t_{i}^{\mathbb{T}},l_{i}^{\mathbb{T}},z,L(0),A)}} \Big]\\
&-{KL}(q(z|\mathcal{D}_{\mathbb{T}},A)\Vert q(z|\mathcal{D}_{\mathbb{C}},A))-{KL}(q(L(0)|\mathcal{D}_{\mathbb{T}},A)\Vert q(L(0)|\mathcal{D}_{\mathbb{C}},A)).
\end{aligned}
\end{equation*}
Since $L(0)$ is only related to $X(0)$ in our design, when the initial states are always known in our implementation, we have $q(L(0)|\mathcal{D}_{\mathbb{C}},A)=q(L(0)|\mathcal{D}_{\mathbb{T}},A)$.
Thus, we can simplify an approximate lower bound ($\mathcal{ELBO}$) as
\begin{equation*}
\centering
\begin{aligned}
&\log{p(X_{l^{\mathbb{T}}_{1:N_{\mathbb{T}}}}(t^{\mathbb{T}}_{1:N_{\mathbb{T}}})\vert t^{\mathbb{T}}_{1:N_{\mathbb{T}}},l^{\mathbb{T}}_{1:N_{\mathbb{T}}},\mathcal{D}_{\mathbb{C}},A)}\\&\geq\mathbb{E}_{q(z|\mathcal{D}_{\mathbb{T}},A)q(L(0)|\mathcal{D}_{\mathbb{T}},A)}\Big[\sum_{i=1}^{N_{\mathbb{T}}}{\log{p(X_{l_{i}^{\mathbb{T}}}(t_{i}^{\mathbb{T}})\vert t_{i}^{\mathbb{T}},l_{i}^{\mathbb{T}},z,L(0),A)}} \Big]\\
&-\beta{KL}(q(z|\mathcal{D}_{\mathbb{T}},A)\Vert q(z|\mathcal{D}_{\mathbb{C}},A)),
\end{aligned}
\end{equation*}
where $\beta$ is a tunable parameter to relieve the KL vanishing.

To learn the distribution over trajectory functions, we train the model on a set of state observations generated from $B$ different network dynamics with $A^{(1)},...,A^{(B)}$ topological structures.
Also, we split the observations into $B$ context sets $\mathcal{D}^{1:B}_{\mathbb{C}}$ and $B$ target sets $\mathcal{D}^{1:B}_{\mathbb{T}}$. 
$\mathcal{D}^{b}_{\mathbb{C}}$ is usually a subset of $\mathcal{D}^{b}_{\mathbb{T}}$.
By accumulating the approximate lower bounds ($\mathcal{ELBO}_{1},\mathcal{ELBO}_{2},...,\mathcal{ELBO}_{B}$) of $B$ network dynamics, the derived total training loss under all observations is as follows
\begin{equation*}
\begin{aligned}
        \mathcal{L}=&\frac{1}{B}{\sum_{b=1}^{B}{-\mathcal{ELBO}_{b}}}\\
        =&\frac{1}{B}\sum_{b=1}^{B}\bigg[\mathbb{E}_{q(z|\mathcal{D}^{b}_{\mathbb{T}},A^{(b)})q(L(0)|\mathcal{D}^{b}_{\mathbb{T}},A)}\bigg(\sum_{i=1}^{N^{b}_{\mathbb{T}}}{-\log{p(X_{l_{i}^{\mathbb{T}_{b}}}(t_{i}^{\mathbb{T}_{b}})\vert t_{i}^{\mathbb{T}_{b}},l_{i}^{\mathbb{T}_{b}},z,L(0),A^{(b)})}}\bigg)
        \\&+\beta{KL}(q(z|\mathcal{D}^{b}_{\mathbb{T}},A^{(b)})\Vert q(z|\mathcal{D}^{b}_{\mathbb{C}},A^{(b)}))\bigg].
\end{aligned}
\end{equation*}
We use reparametrization trick and stochastic gradient descent to minimize the loss.

\subsubsection{The derivation of predictive distribution}

Given the learned NDP4ND, a network topology $A$, and a set of state observations $\mathcal{D}_{\mathbb{C}}
$ from the task at hand,
we obtain the predictive distribution of the states on any node $l^{\mathbb{T}}$ at any time $t^{\mathbb{T}}$ as
\begin{equation*}
\begin{aligned}
p(X_{l^{\mathbb{T}}}(t^{\mathbb{T}})\vert t^{\mathbb{T}},l^{\mathbb{T}},\mathcal{D}_{\mathbb{C}},A)&=\int
q(z|\mathcal{D}_{\mathbb{C}},A)q(L(0)|\mathcal{D}_{\mathbb{C}},A)p(X_{l^{\mathbb{T}}}(t^{\mathbb{T}})\vert t^{\mathbb{T}},l^{\mathbb{T}},z,L(0),A)dzdL(0),
\end{aligned}
\end{equation*}
where, $p(X_{l^{\mathbb{T}}}(t^{\mathbb{T}})\vert t^{\mathbb{T}},l^{\mathbb{T}},z,L(0),A)=\mathcal{N}(\mu_{{X}_{l^{\mathbb{T}}}}(t^{\mathbb{T}}),\sigma^2_{{X}_{l^{\mathbb{T}}}}(t^{\mathbb{T}}))$.
$\mu^{(j)}_{X_{l^{\mathbb{T}}}}(t^{\mathbb{T}})$ and $\sigma^{(j)}_{X_{l^{\mathbb{T}}}}(t^{\mathbb{T}})$ are outputs of decoder network $\bm{d}$.

We can use the Monte Carlo method to approximate the integral and obtain the first and second moments as follows
\begin{equation*}
\begin{aligned}
\mathbb{E}[X_{l^{\mathbb{T}}}(t^{\mathbb{T}})\vert t^{\mathbb{T}},l^{\mathbb{T}},\mathcal{D}_{\mathbb{C}},A]&=\int{\mathbb{E}[X_{l^{\mathbb{T}}}(t^{\mathbb{T}})\vert t^{\mathbb{T}},l^{\mathbb{T}},z,L(0),A]q(z|\mathcal{D}_{\mathbb{C}},A)q(L(0)|\mathcal{D}_{\mathbb{C}},A)dzdL(0)}\\
&=\int{{\mu_{{X}_{l^{\mathbb{T}}}}(t^{\mathbb{T}})}q(z|\mathcal{D}_{\mathbb{C}},A)q(L(0)|\mathcal{D}_{\mathbb{C}},A)}dzdL(0)\\
&\approx \frac{1}{J}\sum_{j=1}^{J}{\mu^{(j)}_{{X}_{l^{\mathbb{T}}}}(t^{\mathbb{T}})},
\end{aligned}
\end{equation*}
\begin{equation*}
\begin{aligned}
\mathbb{E}[X^2_{l^{\mathbb{T}}}(t^{\mathbb{T}})\vert t^{\mathbb{T}},l^{\mathbb{T}},\mathcal{D}_{\mathbb{C}},A]&=\int{\mathbb{E}[X^2_{l^{\mathbb{T}}}(t^{\mathbb{T}})\vert t^{\mathbb{T}},l^{\mathbb{T}},z,L(0),A]q(z|\mathcal{D}_{\mathbb{C}},A)q(L(0)|\mathcal{D}_{\mathbb{C}},A)dzdL(0)}\\
&=\int{{[(\sigma_{{X}_{l^{\mathbb{T}}}}(t^{\mathbb{T}}))^2+(\mu_{{X}_{l^{\mathbb{T}}}}(t^{\mathbb{T}}))^2]}q(z|\mathcal{D}_{\mathbb{C}},A)q(L(0)|\mathcal{D}_{\mathbb{C}},A)}dzdL(0)\\
&\approx \frac{1}{J}\sum_{j=1}^{J}{(\sigma^{(j)}_{{X}_{l^{\mathbb{T}}}}(t^{\mathbb{T}}))^2+(\mu^{(j)}_{{X}_{l^{\mathbb{T}}}}(t^{\mathbb{T}}))^2},
\end{aligned}
\end{equation*}
where $J$ is the sampling number for $z$ and $L(0)$.
We calculate $\mu^{(j)}_{{X}_{l^{\mathbb{T}}}}(t^{\mathbb{T}})$ and $\sigma^{(j)}_{{X}_{l^{\mathbb{T}}}}(t^{\mathbb{T}})$ by sampling $z^{(j)}$ and $L^{(j)}(0)$ from $q(z|\mathcal{D}_{\mathbb{C}},A)$ and $q(L(0)|\mathcal{D}_{\mathbb{C}},A)$ respectively.

Using the moment matching, we can construct a Gaussian posterior approximation for the predictive distribution as
\begin{equation*}
p(X_{l^{\mathbb{T}}}(t^{\mathbb{T}})\vert t^{\mathbb{T}},l^{\mathbb{T}},\mathcal{D}_{\mathbb{C}},A)\approx \mathcal{N}(\mathcal{Z}_1,\mathcal{Z}_2),
\end{equation*}
where, 
\begin{equation*}
\mathcal{Z}_1=\mathbb{E}[X_{l^{\mathbb{T}}}(t^{\mathbb{T}})\vert t^{\mathbb{T}},l^{\mathbb{T}},\mathcal{D}_{\mathbb{C}},A],
\end{equation*}
\begin{equation*}
\mathcal{Z}_2=\mathbb{E}[X^2_{l^{\mathbb{T}}}(t^{\mathbb{T}})\vert t^{\mathbb{T}},l^{\mathbb{T}},\mathcal{D}_{\mathbb{C}},A]-(\mathbb{E}[X_{l^{\mathbb{T}}}(t^{\mathbb{T}})\vert t^{\mathbb{T}},l^{\mathbb{T}},\mathcal{D}_{\mathbb{C}},A])^2.
\end{equation*}

\subsection{Proofs}

In the NDP4ND, given $\tilde{z}$ and $L_{l^{\mathbb{T}}_{i}}(0)$, the evolved latent state $L_{l^{\mathbb{T}}_{i}}(t^{\mathbb{T}}_i)$ can be seen as a deterministic function dominated by $t^{\mathbb{T}}_i$ and $l^{\mathbb{T}}_{i}$, involving the interaction between nodes.
Thus, we have the following Lemma to assist in proving our proposition.

\begin{lemma}
    The evolved latent state $L_{l^{\mathbb{T}}_{i}}(t^{\mathbb{T}}_i)$ can be seen as a deterministic function $\mathcal{F}(t^{\mathbb{T}}_i,l^{\mathbb{T}}_{i})$ for a given fixed $\tilde{z}$ and $L_{l^{\mathbb{T}}_{i}}(0)$.
\label{lemma2}
\end{lemma}
\begin{proof}
We have the fact that 
\begin{equation*}
        L_{l^{\mathbb{T}}_{i}}(t^{\mathbb{T}}_i)=L_{l^{\mathbb{T}}_{i}}(0)+\int_{t_0}^{t^{\mathbb{T}}_i}{\bigg( \bm{S}(L_{l^{\mathbb{T}}_i}(t),\tilde{z})+\sum_{j=1}^{n}A_{l^{\mathbb{T}}_i,l^{\mathbb{T}}_j}\bm{I}(L_{l^{\mathbb{T}}_i}(t),L_{l^{\mathbb{T}}_j}(t),\tilde{z})}\bigg)dt,
\end{equation*}
where, only $\tilde{z}$ (global control variable) and $L_{l^{\mathbb{T}}_{i}}(0)$ (initial hidden states) are uncertain, while the rest are deterministic.
Therefore, when both $\tilde{z}$ and $L_{l^{\mathbb{T}}_{i}}(0)$ are fixed, the above equation can be regarded as a deterministic function controlled by $t^{\mathbb{T}}_i$ and $l^{\mathbb{T}}_{i}$, i.e., $\mathcal{F}(t^{\mathbb{T}}_i,l^{\mathbb{T}}_{i})$.
\end{proof}

\begin{proposition}
    NDP4ND satisfies the exchangeability and consistency conditions.
\label{prop1}
\end{proposition}
\begin{proof}
    First, we prove that NDP4ND satisfies the exchangeability condition.
    Following Lemma~\ref{lemma2}, for any given $\tilde{z}$ and $L_{l^{\mathbb{T}}_{i}}(0)$, any permutation $\pi$ of $\{1,2,...,N_{\mathbb{T}}\}$ on $t^{\mathbb{T}}_{1:N_{\mathbb{T}}}$ and $l^{\mathbb{T}}_{1:N_{\mathbb{T}}}$ would automatically act on $\mathcal{F}_{1:N_{\mathbb{T}}}$, i.e.,
    \begin{equation*}
        \mathcal{F}_{\pi(1:N_{\mathbb{T}})}=\mathcal{F}(t^{\mathbb{T}}_{\pi(1:N_{\mathbb{T}})},l^{\mathbb{T}}_{\pi(1:N_{\mathbb{T}})}),
    \end{equation*}
    where $\pi(1:N_{\mathbb{T}})=(\pi(1),\pi(2),...,\pi(N_{\mathbb{T}}))$.
    And then the permutation $\pi$ would act on the conditional predictive distribution, i.e.,
    \begin{equation*}
        [\mu_{{X}_{l^{\mathbb{T}}_{\pi(1:N_{\mathbb{T}})}}}(t^{\mathbb{T}}_{\pi(1:N_{\mathbb{T}})}); \sigma^2_{{X}_{l^{\mathbb{T}}_{\pi(1:N_{\mathbb{T}})}}}(t^{\mathbb{T}}_{\pi(1:N_{\mathbb{T}})})]=\bm{d}(t^{\mathbb{T}}_{\pi(1:N_{\mathbb{T}})},\mathcal{F}_{\pi(1:N_{\mathbb{T}})},\tilde{z}).
    \end{equation*} 
    Consequently, the permutation $\pi$ would act on the predictive distribution, i.e., $
p(X_{l^{\mathbb{T}}_{\pi(1:N_{\mathbb{T}})}}(t^{\mathbb{T}}_{\pi(1:N_{\mathbb{T}})})\vert t^{\mathbb{T}}_{\pi(1:N_{\mathbb{T}})},l^{\mathbb{T}}_{\pi(1:N_{\mathbb{T}})},\mathcal{D}_{\mathbb{C}},A)
$, because $\mathcal{Z}_1$ and $\mathcal{Z}_2$ in the distribution can also be seen as two deterministic functions dominated by $t^{\mathbb{T}}_i$ and $l^{\mathbb{T}}_{i}$.
Therefore, the exchangeability condition is guaranteed.

Then, we prove that NDP4ND satisfies the consistency condition. Based on Lemma~\ref{lemma2}, we can write the joint distribution similar to NPs \citep{DBLP:journals/corr/abs-1807-01622} or NDPs \citep{DBLP:conf/iclr/NorcliffeBDML21} as follows
\begin{equation*}
    \rho_{t^{\mathbb{T}}_{1:N_{\mathbb{T}}},l^{\mathbb{T}}_{1:N_{\mathbb{T}}}}(X_{l^{\mathbb{T}}_{1:N_{\mathbb{T}}}}(t^{\mathbb{T}}_{1:N_{\mathbb{T}}}))=\int{p(\mathcal{F})\prod_{i=1}^{N_{\mathbb{T}}}{p(X_{l^{\mathbb{T}}_{i}}(t^{\mathbb{T}}_{i})|\mathcal{F}(t^{\mathbb{T}}_{i},l^{\mathbb{T}}_{i}))}}d\mathcal{F}.
\end{equation*}
Since the probability density function of any $X_{l^{\mathbb{T}}_{i}}(t^{\mathbb{T}}_{i})$ depends only on the corresponding pair $t^{\mathbb{T}}_{i}$ and $l^{\mathbb{T}}_{i}$, integrating out any subset of $X_{l^{\mathbb{T}}_{1:N_{\mathbb{T}}}}(t^{\mathbb{T}}_{1:N_{\mathbb{T}}})$ gives the joint distribution of the remaining random variables in the sequence as 
\begin{equation*}
\begin{aligned}
    &\int\rho_{t^{\mathbb{T}}_{1:N_{\mathbb{T}}},l^{\mathbb{T}}_{1:N_{\mathbb{T}}}}(X_{l^{\mathbb{T}}_{1:N_{\mathbb{T}}}}(t^{\mathbb{T}}_{1:N_{\mathbb{T}}}))dX_{l^{\mathbb{T}}_{n+1:N_{\mathbb{T}}}}(t^{\mathbb{T}}_{n+1:N_{\mathbb{T}}})\\
    &=\int\int{p(\mathcal{F})\prod_{i=1}^{N_{\mathbb{T}}}{p(X_{l^{\mathbb{T}}_{i}}(t^{\mathbb{T}}_{i})|\mathcal{F}(t^{\mathbb{T}}_{i},l^{\mathbb{T}}_{i}))}}d\mathcal{F}dX_{l^{\mathbb{T}}_{n+1:N_{\mathbb{T}}}}(t^{\mathbb{T}}_{n+1:N_{\mathbb{T}}})\\
    &=\int{p(\mathcal{F})\prod_{i=1}^{n}{p(X_{l^{\mathbb{T}}_{i}}(t^{\mathbb{T}}_{i})|\mathcal{F}(t^{\mathbb{T}}_{i},l^{\mathbb{T}}_{i}))}}d\mathcal{F}\\
    &=\rho_{t^{\mathbb{T}}_{1:n},l^{\mathbb{T}}_{1:n}}(X_{l^{\mathbb{T}}_{1:n}}(t^{\mathbb{T}}_{1:n})).
\end{aligned}
\end{equation*}
Therefore, consistency is also guaranteed.
\end{proof}

\section{Experimental Details}

In this section, we introduce specific experimental details, including the construction for training and testing sets, 
the detailed description of the network dynamics scenarios, baselines, metrics, and experimental setup.

\subsection{Construction for training and testing sets}

We assign randomly sampled values from the parameter to the tunable parameters in the ODE templates to obtain massive network ODE instances for each network dynamics scenario.
Based on different ODE instances, topologies, and initial states, we randomly generate $N_{Tr}$ training tasks and $N_{Te}$ testing tasks per scenario.
Note that the sampled ODE instances to produce testing tasks have not appeared in the training set.
Therefore, for the model, the testing tasks to be handled come from the emerging dynamics.
We irregularly and sparsely sample observations from each testing task to trigger predictions (Fig.~\ref{fig_obs}).

\begin{figure}[t]%
\centering
\includegraphics[width=.6\textwidth]{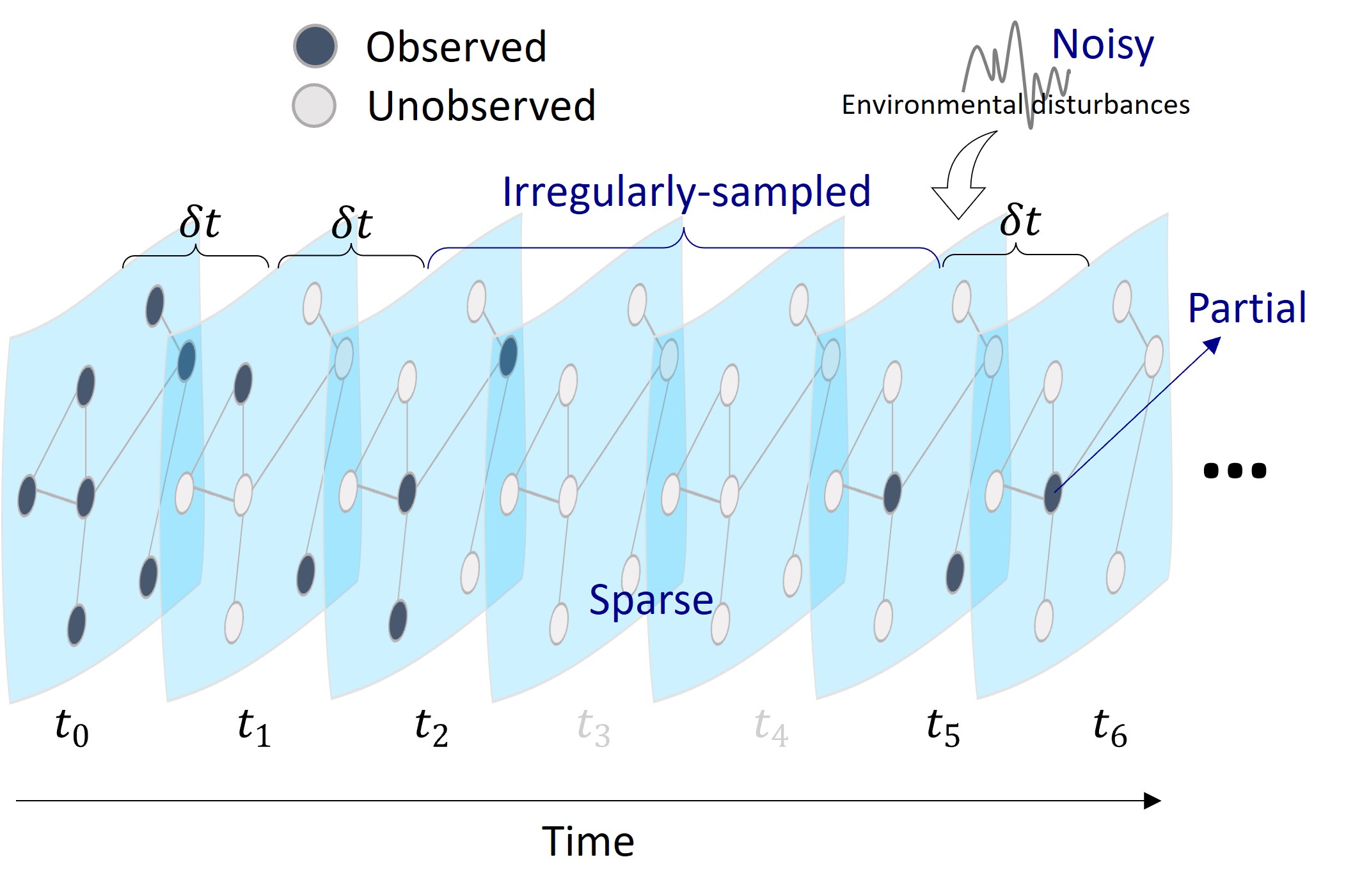}
\caption{
An illustration to show that the observations triggering predictions are sparse, irregularly-sampled, partial, and noisy.
}\label{fig_obs}
\end{figure}

\subsection{Network dynamics scenarios}

We carried out extensive experiments on various complex network dynamics scenarios, including mutualistic interaction dynamics in ecosystems, second-order phototaxis dynamics, brain dynamics, compartment models in epidemiology, and empirical systems.

\subsubsection{Mutualistic interacting dynamics}
Mutualistic interacting dynamics among species in ecology \citep{JianxiNature2016} capture the abundance $X_i(t)$ of species $i$.
The differential equation that characterizes abundance changes is as follows
\begin{equation*}
   \begin{aligned} 
   \dot{X}_i(t)=b+X_i(t)(1-\frac{X_i(t)}{k})(\frac{X_i(t)}{c}-1)+\sum_jA_{i,j} (X_j(t)-X_i(t))\frac{X_i(t)X_j(t)}{d+eX_i(t)+hX_j(t)},
   \end{aligned}
\end{equation*}
where $b$ accounts for the rate of the incoming migration from neighboring ecosystems. 
The second term describes logistic growth with the system carrying capacity $k$ and the Allee effect with negative growth threshold $c$. 
The third term describes mutualistic interactions, captured by a response function that saturates for large $X_i(t)$ or $X_j(t)$, indicating that $j$’s positive contribution to $X_i(t)$ is bounded.

The detailed parameter setting for mutualistic interacting dynamics in epidemiology is shown in Table~\ref{table_set_mutu}.

\begin{table}[htb]
\centering
\caption{
Parameter setting for Mutualistic interacting dynamics
}
\label{table_set_mutu}
\begin{tabular}{cccc}
\toprule
Network space $\mathcal{A}$                                            &   Parameter space $\Phi$         &  \# Nodes $n$  &  Initial state space $\mathcal{I}$ \\\midrule
\makecell[c]{Grid,\\ Power law,\\ Random,\\ Small world,\\ Community}  & $\begin{aligned} &b\in [0.05, 0.15],\\ &c\in [0.5, 1.5],\\ &d\in [4, 6], \\ &e\in [0.8, 1.0],\\ &h\in [0.05, 0.15],\\ &k\in [4, 6].\end{aligned}$  &    $\mathcal{U}^{\dagger}[100,200]$&   $\mathcal{U}[0,25]$ \\\midrule\midrule
\# Training tasks $N_{Tr}$ &   \# Testing tasks $N_{Te}$   & Min. sampling interval $\delta t$ & Max. observed time $T_o$ \\\midrule
1,500                      &   100                         &        1                          &             50              \\
\bottomrule                
\end{tabular}
\footnotetext[\dagger]{$\mathcal{U}$ means the uniform distribution.}
\end{table}

\subsubsection{Second-order phototaxis dynamics}

Phototaxis dynamics is a second-order system simulating the dynamics of phototactic bacteria towards a fixed light source \citep{phototaxis}, which can be applied to movement of bacteria towards food sources \citep{doi:10.1073/pnas.1822012116} and emergency evacuation modeling \citep{emergency_airplane_evacuations}.
The system performs bacteria's movements through the excitation level of bacteria to the light and their interactions with each other \citep{doi:10.1073/pnas.1822012116}.
The system has five observed states: coordinate 1 (${X}_{i,0}(t)$), coordinate 2 (${X}_{i,1}(t)$), velocity 1 (${X}_{i,2}(t)$), velocity 2 (${X}_{i,3}(t)$) and excitation level (${X}_{i,4}(t)$).
The differential equations that characterize system behavior are as follows
\begin{equation*}
    \begin{aligned} 
        \dot{X}_{i,0:1}(t)=&{X}_{i,2:3}(t),\\
        \dot{X}_{i,2:3}(t)=&\frac{\lambda_1}{n}\sum_{j=1}^{n}k_1({X}_{j,0:1}(t),{X}_{i,0:1}(t))({X}_{j,2:3}(t)-{X}_{i,2:3}(t))\\&+I_0(V-{X}_{i,2:3}(t))(1-\varphi({X}_{i,4}(t);\zeta_{cr})),\\
        \dot{X}_{i,4}(t)=&\frac{\lambda_2}{n}\sum_{j=1}^{n}k_2({X}_{j,0:1}(t),{X}_{i,0:1}(t))({X}_{j,4}(t)-{X}_{i,4}(t))+I_0\varphi({X}_{i,4}(t);\zeta_{cp}),
    \end{aligned}
\end{equation*}
where, 
\begin{equation*}
    \begin{aligned} 
        &k_1(x_j,x_i)=k_2(x_j,x_i)=\frac{1}{(1+\lvert x_j-x_i\rvert^2)^{\beta}},\\
        &\varphi(\zeta;\zeta_{c})= \left\{
        \begin{aligned}
	&1, \quad &0\leq\zeta<\zeta_{c},\\
	&0.5(\cos{(\frac{\pi}{\zeta_{c}}(\zeta-\zeta_{c}))}+1), \quad &\zeta_{c}\leq\zeta<2\zeta_{c},\\
	&0, \quad &2\zeta_{c}\leq\zeta.\\
	\end{aligned}
	\right.
    \end{aligned}
\end{equation*}
$k_1$ and $k_2$ are the alignment-based interaction kernels, 
$\varphi$ is the smooth cutoff function,
$I_0$ is the light intensity, $V$ is the terminal velocity (light source at infinity), $\zeta_{cr}$ is the critical excitation level (when the light effect activates the bacteria), and $\zeta_{cp}$ is the excitation capacity.

The detailed parameter setting for the second-order phototaxis dynamics is shown in Table~\ref{table_set_phototaxis}.

\begin{table}[htb]
\centering
\caption{
Parameter setting for Second-order phototaxis dynamics
}
\label{table_set_phototaxis}
\begin{tabular}{cccc}
\toprule
Network space $\mathcal{A}$                                            &   Parameter space $\Phi$         &  \# Nodes $n$  &  Initial state space $\mathcal{I}$ \\\midrule
Complete graph  & $\begin{aligned} &\lambda_1=\lambda_2=100,\\&V=[60, 0],\\ &\zeta_{cp}=2\zeta_{cr},\\ &I_0\in [0.01, 1.0], \\ &\beta\in [-0.4, -0.1],\\ &\zeta_{cr}\in [0.1, 0.5].\end{aligned}$  &    $\mathcal{U}^{\dagger}[20,50]$&   \makecell[c]{$X_{i,0:3}\sim\mathcal{U}[0,100]^4$ \\$X_{i,4}\sim\mathcal{U}[0,\text{1e-3}]$ }\\\midrule\midrule
\# Training tasks $N_{Tr}$ &   \# Testing tasks $N_{Te}$   & Min. sampling interval $\delta t$ & Max. observed time $T_o$ \\\midrule
300                      &   20                         &        0.01                         &             0.5              \\
\bottomrule                
\end{tabular}
\footnotetext[\dagger]{$\mathcal{U}$ means the uniform distribution.}
\end{table}

\subsubsection{Brain dynamics}

We consider a kind of neuronal dynamics dynamics that brain activities are governed by the FitzHugh-Nagumo dynamics \citep{FHN2020}, which has a quasi-periodic characteristic.
Topological structures with power-law distribution are used to simulate the interactions among brain functional areas \citep{Gao2022}.
The dynamics capture the firing behaviors of neurons with two components, i.e., membrane potential ($X_{i,0}(t)$) and recovery variable ($X_{i,1}(t)$).
The differential equations that characterize brain activities are as follows
\begin{equation*}
    \begin{aligned} 
        \dot{X}_{i,0}(t)=&{X}_{i,0}(t)-{X}_{i,0}(t)^{3}-{X}_{i,1}(t)-\epsilon\sum_{j=1}^{n}A_{i,j}\frac{{X}_{j,0}(t)-{X}_{i,0}(t)}{Deg_{i}^{in}},\\
        \dot{X}_{i,1}(t)=&a+b{X}_{i,0}(t)+c{X}_{i,1}(t),\\
    \end{aligned}
\end{equation*}
where, $Deg_{i}^{in}$ is the indegree of node $i$. 

The detailed parameter setting for brain dynamics is shown in Table~\ref{table_set_brain}.
\begin{table}[htb]
\centering
\caption{
Parameter setting for Brain dynamics
}
\label{table_set_brain}
\begin{tabular}{cccc}
\toprule
Network space $\mathcal{A}$                                            &   Parameter space $\Phi$         &  \# Nodes $n$  &  Initial state space $\mathcal{I}$ \\\midrule
Power law & $\begin{aligned} &a\in [0.2, 0.3],\\&b\in [0.4, 0.6],\\ &c\in [-0.06, -0.02],\\ &\epsilon\in [0.8, 1.2].\end{aligned}$  &    $\mathcal{U}^{\dagger}[100,200]$&   $\mathcal{U}[-1,1]^{2}$ \\\midrule\midrule
\# Training tasks $N_{Tr}$ &   \# Testing tasks $N_{Te}$   & Min. sampling interval $\delta t$ & Max. observed time $T_o$ \\\midrule
300                      &   20                         &        0.01                          &             0.5              \\
\bottomrule                
\end{tabular}
\footnotetext[\dagger]{$\mathcal{U}$ means the uniform distribution.}
\end{table}

\subsubsection{Compartment models in epidemiology}

Compartment models \citep{doi:10.1287/educ.1100.0075,individualSIR} are often used to simulate the spread of epidemics on networks. 
Herein, we conduct experiments on three commonly used compartment models: susceptible-infectious-susceptible (SIS), susceptible-infectious-recovered (SIR), and susceptible-exposed-infectious-susceptible (SEIS).

\textbf{SIS}: 
The susceptible-infectious-susceptible (SIS) model has two states, i.e., susceptible and infectious.
This model assumes that individual states not only transition from susceptible to infectious with an infection rate $b$ but also reverse with a recovery rate $r$.
The differential equations that characterize the model are as follows
\begin{equation*}
    \begin{aligned} 
        &S: \dot{X}_{i,0}(t)=-b\times{X}_{i,0}(t)\times\sum_j A_{i,j}{X}_{j,1}(t)+r\times{{X}_{i,1}(t)},\\
        &I: \dot{X}_{i,1}(t)=b\times{X}_{i,0}(t)\times\sum_j A_{i,j}{X}_{j,1}(t)-r\times{{X}_{i,1}(t)}.
    \end{aligned}
\end{equation*}

\textbf{SIR}:
The susceptible-infectious-recovered (SIR) model has three states, i.e., susceptible, infectious, and recovered.
The recovered state refers to an individual who has recovered from an epidemic.
It assumes that 1) individual states transition from susceptible to infectious with an infection rate $b$; 2) individual states transition from infectious to recovered with a recovery rate $r$.
In this model, the recovered individual has immunity, so he cannot be infected again.
The differential equations that characterize the model are as follows
\begin{equation*}
    \begin{aligned} 
        &S: \dot{X}_{i,0}(t)=-b\times{X}_{i,0}(t)\times\sum_j A_{i,j}{X}_{j,1}(t),\\
        &I: \dot{X}_{i,1}(t)=b\times{X}_{i,0}(t)\times\sum_j A_{i,j}{X}_{j,1}(t)-r\times{X}_{i,1}(t),\\
        &R: \dot{X}_{i,2}(t)=r\times{X}_{i,1}(t).
    \end{aligned}
\end{equation*}

\textbf{SEIS}:
The susceptible-exposed-infectious-susceptible (SEIS) model has three states: susceptible, exposed, and infectious.
The exposed state refers to an individual who has come into contact with infected individuals but is currently not contagious.
The model assumes that 1) individual states transition from susceptible to exposed with an infection rate $b$; 2) individual states transition from exposed to infectious with an attack rate $c$; 3) individual states transition from infectious to susceptible with a recovery rate $r$.
The differential equations that characterize the model are as follows
\begin{equation*}
    \begin{aligned} 
        &S: \dot{X}_{i,0}(t)=-b\times{X}_{i,0}(t)\times\sum_j A_{i,j}{X}_{j,1}(t)+r\times{{X}_{i,1}(t)},\\
        &E: \dot{X}_{i,1}(t)=b\times{X}_{i,0}(t)\times\sum_j A_{i,j}{X}_{j,1}(t)-c\times{{X}_{i,1}(t)},\\
        &I: \dot{X}_{i,2}(t)=c\times{{X}_{i,1}(t)}-r\times{{X}_{i,2}(t)}.
    \end{aligned}
\end{equation*}

The detailed parameter settings for compartment models in epidemiology are shown in Table~\ref{table_set_compartment}.

\begin{table}[htb]
\centering
\caption{
Parameter setting for Compartment models in epidemiology (SIS, SIR, and SEIS)
}
\label{table_set_compartment}
\begin{tabular}{cccc}
\toprule
Network space $\mathcal{A}$                                            &   Parameter space $\Phi$         &  \# Nodes $n$  &  Initial state space $\mathcal{I}$ \\\midrule
Power law  & 
\makecell[l]{
\makecell[l]{For SIS and SIR:}\\
\makecell[r]{
$\begin{aligned} &b\in [0.02, 0.2], \\ &r\in [0.1, 0.4].\end{aligned}$}\\
\makecell[l]{For SEIS:}\\
\makecell[r]{
$\begin{aligned} &b\in [0.3, 2.0], \\ &r\in [0.1, 0.4], \\ &c\in[0.05,0.1].\end{aligned}$}}  &    $\mathcal{U}^{\dagger}[100,200]$& \makecell[l]{
\makecell[l]{For SIS:}\\
\makecell[l]{$\begin{aligned} &S: X_{i,0}=1-X_{i,1},\\ &I: X_{i,1}\sim\mathcal{U}[\text{1e-6},\text{1e-3}]\end{aligned}$}\\
\makecell[l]{For SIR:}\\
\makecell[l]{$\begin{aligned} &S: X_{i,0}=1-X_{i,1}-X_{i,2}, \\&I: X_{i,1}\sim\mathcal{U}[\text{1e-6},\text{1e-3}], \\&R: X_{i,2}=0\end{aligned}$}\\
\makecell[l]{For SEIS:}\\
\makecell[l]{$\begin{aligned} &S: X_{i,0}=1-X_{i,1}-X_{i,2},\\ &E: X_{i,1}=0, \\ &I: X_{i,2}\sim\mathcal{U}[\text{1e-6},\text{1e-3}]\end{aligned}$}
}  \\\midrule\midrule
\# Training tasks $N_{Tr}$ &   \# Testing tasks $N_{Te}$   & Min. sampling interval $\delta t$ & Max. observed time $T_o$ \\\midrule
1,000 per model                      &   100 per model                       &        1                          &             50              \\
\bottomrule                
\end{tabular}
\footnotetext[\dagger]{$\mathcal{U}$ means the uniform distribution.}
\end{table}

\subsubsection{Empirical systems}

For the real-world systems, we collect daily global spreading data on three real diseases, including 
H1N1 \citep{dataweb},
SARS \citep{dataweb},
and COVID-19 \citep{covid19data}, and use the worldwide airline network retrieved from OpenFights as a directed and weighted empirical topology.
Only early data before government intervention, i.e., the first 50 days, is considered here to keep the spread features of the disease itself.
We divide H1N1 data into 1,000 standardized training tasks according to various time intervals, where the time interval is sampled from $\{2,3,...,21\}$ and the start time is sampled from $\{0,1,...,49\}$.
Then, we test methods on all three diseases.
As a result, testing on H1N1 can be seen as predicting re-emerging infectious diseases, as the method trained on the same data but testing on SARS and COVID-19 as predicting emerging infectious diseases.

\subsection{Baselines}

The methods we compared in the experiments are listed below.

\begin{itemize}
            \item LG-ODE \citep{NEURIPS2020_ba484941}: The first work is to learn continuous-time system dynamics from irregularly-sampled partial observations,
            which is to use a carefully-designed spatio-temporal transformer as the encoder and use a graph neural network to model the network dynamic equation.
            \item NDCN \citep{NDCN}: A representative neural dynamics method, explicitly proposed for automatically modeling network dynamics. Its main idea is to use a graph neural network to model the network dynamic equation in the hidden space mapped by the neural encoder, and can learn continuous-time dynamics on complex networks with irregularly-sampled observations.
            \item DNND \citep{ijcai2023p242}: A state-of-the-art neural dynamics method for automatically modeling network dynamics. It
            alleviates the limitations of the NDCN in long-term prediction by abandoning the encoder and decoder, directly models network differential equations in the state space, and simplifies the design of neural networks using first principles, achieving state-of-the-art results.
            \item NP \citep{DBLP:journals/corr/abs-1807-01622}: Stochastic processes parameterized by neural networks.
            \item NDP \citep{DBLP:conf/iclr/NorcliffeBDML21}: A class of stochastic processes, generalized NPs defined over time by combining latent neural ODEs with NPs.
            \item NDP4ND w/o ode: A variant of our NDP4ND, removing ODE flow in the architecture, i.e., deleting $L_{l_{i}^{\mathbb{T}}}(t_{i}^{\mathbb{T}})$ from the input of neural network $\bm{d}$.
            \item NDP4ND w/o z: A variant of our NDP4ND, removing $z$ from the input of neural network $\bm{d}$ in the architecture.
\end{itemize}

We note that NPs and NDPs are not specifically designed for network dynamics learning.
Using them as comparison methods in experiments is to verify the necessity of introducing interactions of time dynamics on networks.

\subsection{Metrics}
The following metrics are used for performance comparison: 

\begin{itemize}
    \item Mean Absolute Error (MAE): MAE quantifies the mean absolute error between the predictions and the true values over all nodes and at all times, which can be calculated as
    \begin{equation*}
    \begin{aligned} 
        \frac{1}{nTD}\sum_{i=1}^{n}\sum_{t=1}^{T}\sum_{d=1}^{D}|X_{i}(t)[d]-\hat{X}_{i}(t)[d]|,
    \end{aligned}
    \end{equation*}
    where, $n$, $T$, and $D$ are the number of nodes in the system, the maximum prediction time, and the dimension of observable states, respectively.
    $\hat{X}_{i}(t)[d]$ is the $d$-th dimension predictive state on node $i$ at time $t$ and $X_{i}(t)[d]$ is the ground truth.
    \item Dynamic Time Warping (DTW): DTW quantifies the similarity between predictive system dynamic behavior and ground truth, which can be calculated as
    \begin{equation*}
    \begin{aligned} 
        \frac{1}{nD}\sum_{i=1}^{n}\sum_{d=1}^{D}dtw(X_{i}(.)[d]-\hat{X}_{i}(.)[d]),
    \end{aligned}
    \end{equation*}
    where, $n$ and $D$ are the number of nodes in the system and the dimension of observable states, respectively.
    $dtw(.)$ is a FastDTW method \citep{Salvador2004FastDTWTA} that is an approximation of dynamic time warping with linear time and space complexity.
    \item Observed ratio: We define the observed ratio as $\frac{N_{\mathbb{C}}}{((t_{o}-t_{s})/\delta t+1)\times \#nodes}\times 100\%$ to quantify the sampling density of observations, where $N_{\mathbb{C}}$ is the number of observations, $\delta t$ is the minimum sampling interval, $t_{s}$ is the start time, and $t_{o}$ is maximum time of all observations.
\end{itemize}

\subsection{Experimental setup}

In the training phase, we use the Adam algorithm \citep{adam} to optimize parameters in the model.
Specifically, the initial learning rate is set to 5e-3, followed by an exponential decay with a decay coefficient of 10.
The batch size and epochs are set to 8 and 200, respectively.
All experiments were conducted on the hardware environment with the same computational resources, including Intel(R) Xeon(R) Silver 4310 CPU@2.10GHz $\times$ 48 and NVIDIA GeForce RTX 3090 $\times$ 3.

\end{appendix}





\bibliography{supplement}